\documentclass[10pt,letterpaper]{article}
\usepackage[preprint]{boltgan_tmlr}
\usepackage[utf8]{inputenc}
\usepackage[T1]{fontenc}
\usepackage{amsmath, amssymb, amsfonts, amsthm, mathtools, bbm}
\usepackage{enumitem}
\usepackage{needspace} 
\usepackage{mathrsfs} 
\usepackage{xspace}

\usepackage{booktabs, makecell, tabularx, enumitem, array}
\usepackage{multirow}
\newcolumntype{Y}{>{\raggedright\arraybackslash}X}
\newcolumntype{Z}{>{\centering\arraybackslash}X}
\newcommand{\MainTableSetup}{\small\setlength{\tabcolsep}{3.0pt}\renewcommand{\arraystretch}{1.0}}

\usepackage{algorithm}
\usepackage{algorithmic}

\usepackage{graphicx}
\usepackage{microtype}
\usepackage{subcaption}      
\usepackage{tikz, pgfplots}
\usepackage{url}
\usetikzlibrary{fillbetween}
\pgfplotsset{compat=1.18}
\usepgfplotslibrary{fillbetween}

\usepackage{xcolor, url, etoolbox, placeins, float}
\AtBeginEnvironment{algorithmic}{\let\AND\relax}
\usepackage{natbib}
\usepackage{hyperref}

\usepackage[nolist]{acronym}

\DeclareMathOperator{\Var}{Var}

\DeclareMathOperator*{\argmax}{arg\,max}
\DeclareMathOperator*{\argmin}{arg\,min}

\newtheorem{theorem}{Theorem}
\newtheorem{definition}{Definition}
\newtheorem{lemma}{Lemma}
\newtheorem{corollary}{Corollary}

\newtheorem{proposition}{Proposition}

\newtheorem{remark}{Remark}
\newtheorem*{theorem*}{Theorem}

\newcommand{\defeq}{\vcentcolon=}
\newcommand{\E}{\mathbb{E}}

\newcommand{\Bayes}{\varepsilon_{\mathrm{Bayes}}}

\newcommand{\Lip}{\mathrm{Lip}}

\newcommand{\BOLTGAN}{\textsc{BOLT-GAN}\xspace} 
\newcommand{\BOLTGANNaive}{\textsc{Max-BER}\xspace} 
\newcommand{\WGANGP}{\textsc{WGAN-GP}\xspace}
\newcommand{\HingeSN}{\textsc{HINGE-GAN}\xspace}

\newcommand{\PRn}{2000} 
\newcommand{\PRk}{3}    

\hypersetup{hidelinks,pdftitle={Revisiting GAN with Bayes-Optimal Discrimination}}
\title{Revisiting GAN with Bayes-Optimal Discrimination}
\author{%
\name Mohammadreza Tavasoli Naeini \email mohammadreza.tavasolinaeini@mail.utoronto.ca \\
\addr University of Toronto
\AND
\name Ali Bereyhi \email ali.bereyhi@utoronto.ca \\
\addr University of Toronto
\AND
\name Morteza Noshad \email noshad@stanford.edu \\
\addr Stanford University
\AND
\name Ben Liang \email liang@ece.utoronto.ca \\
\addr University of Toronto
\AND
\name Alfred O. Hero III \email hero@eecs.umich.edu \\
\addr University of Michigan
}
\hypersetup{pdfauthor={Mohammadreza Tavasoli Naeini, Ali Bereyhi, Morteza Noshad, Ben Liang, Alfred O. Hero III}}

\allowdisplaybreaks
\begin{document}
\maketitle
\begin{abstract}
We approach generative adversarial network (GAN) training from its original real-versus-fake classification viewpoint, replacing the cross-entropy discriminator objective with an upper bound on the Bayes error rate (BER).
Using the Bayes optimal learning threshold (BOLT) loss, we formulate an adversarial game in which the discriminator tightens this bound and the generator maximizes it.
We characterize the unrestricted optimal discriminator and show that, under balanced class priors, the population game is equivalent to minimizing total variation between the data and generator distributions.
Restricting the bounded discriminator to be $1$-Lipschitz yields BOLT-GAN.
For balanced priors and the full discriminator class, we establish upper and lower bounds relating its population discrepancy to Wasserstein-$1$ distance on bounded domains.
Controlled experiments on four image-generation benchmarks show improved sample quality and coverage over WGAN-GP under matched architectures, optimization settings, and gradient penalties.
At a common $20$-epoch budget, BOLT-GAN reduces Fr\'echet Inception Distance (FID) from $60.0$ to $44.2$ on CIFAR-10 and from $102.5$ to $28.4$ on LSUN Bedroom-64.
\end{abstract}

\section{Introduction}
Generative adversarial networks (GANs) train a discriminator to distinguish real data from generated samples, while a generator tries to make this distinction difficult. The original discriminator uses cross-entropy~\citep{goodfellow2014generative}. We ask how the adversarial game changes when its classification criterion is instead a surrogate of the minimum achievable discrimination error, the Bayes error rate (BER).

The Bayes optimal learning threshold (BOLT) framework~\citep{naeini2025universal,tavasoli2026bolt} gives an upper bound on BER using a bounded score. Minimizing its binary population risk over all measurable scores yields an endpoint-valued MAP decision score. In supervised classification the two class-conditional laws are fixed. In a GAN, one is the generated distribution $P_{\mathrm g}$, which changes after every generator update. We therefore derive a two-player game: the discriminator tightens the BER bound for the current $P_{\mathrm g}$, and the generator changes $P_{\mathrm g}$ to increase that bound.

At balanced priors, the unrestricted optimized objective equals half the total variation (TV) distance. Its optimal discriminator can switch abruptly between $0$ and $1$. Following WGAN~\citep{arjovsky2017wasserstein}, we restrict the bounded score to a Lipschitz class. The resulting objective is a bounded WGAN-style integral probability metric (IPM), with a real-minus-fake expectation structure. Our contribution is its derivation from BER, the analysis of its unrestricted and Lipschitz-constrained forms, and its evaluation under matched neural training. The bounded score gives an explicit BER scale and distribution-free control of fixed-critic sampling variance. Appendix~\ref{app:gradient-tradeoff} further characterizes how the sigmoid parameterization scales generator gradients.

\subsection{Main Contributions}
\textbf{(i)} We derive an adversarial game from the BOLT discrimination risk and characterize MAP recovery over unrestricted scores.
\textbf{(ii)} We relate the balanced unrestricted game to TV and the bounded Lipschitz game to Wasserstein distance. We also bound the ratio-approximation error and conditional Monte Carlo variance of a density-ratio BER estimator.
\textbf{(iii)} We compare BOLT-GAN and WGAN-GP using the same architectures, optimization settings, and gradient penalty. Additional experiments cover other datasets, discriminator variants, mode coverage, and conditional generation; the SN-based HINGE-GAN comparison is presented separately.

\subsection{Related Work}
\paragraph{Bayes-error estimation.} The BER is the minimum classification error achievable by any measurable classifier~\citep{bishop2006pattern,devroye1996probabilistic}. Sample-based approaches include $\epsilon$-ball ensemble BER estimation~\citep{noshad2019learning}, weighted-neighbor bounds~\citep{noshad2018rate}, multiclass bounds~\citep{sekeh2020learning}, and direct estimation from uncertain class labels~\citep{ishida2023performance}. Related divergence estimators use optimally weighted ensembles~\citep{moon2018information}, nearest-neighbor ratios~\citep{noshad2017direct}, or variational convex-risk minimization~\citep{nguyen2010estimating}. These methods focus on estimating divergences or classification limits. Using such estimates to train a generator requires an adversarial objective. The BOLT framework~\citep{naeini2025universal} instead gives a universal BER upper bound that can be estimated by sampling a bounded scoring model. Minimizing its empirical counterpart trains the score toward a Bayes-optimal decision rule in the binary setting considered here.

\paragraph{GAN training objectives.} GANs were introduced as a min-max game between a generator and a discriminator~\citep{goodfellow2014generative}. The discriminator learns to distinguish real from generated samples, while the generator learns to make that distinction harder. The original cross-entropy objective in vanilla GAN corresponds to the Jensen-Shannon divergence and is known to suffer from saturation and instability. Many alternatives replace or relax the original real-versus-fake classification objective with other discrepancies, including $f$-divergences~\citep{nowozin2016f}, least-squares losses~\citep{mao2017least}, hinge losses~\citep{miyato2018spectral}, energy-based objectives~\citep{berthelot2017began}, and relativistic formulations~\citep{jolicoeur2019relativistic}. WGAN and its gradient-penalty variant~\citep{arjovsky2017wasserstein,gulrajani2017improved} replace the probabilistic discriminator with a real-valued $1$-Lipschitz critic for estimating the Wasserstein-1 distance and have become standard stable baselines. Regularization-based analyses study local convergence of GAN training~\citep{mescheder2018which}, while R3GAN combines a relativistic paired objective with both zero-centered $R_1$ and $R_2$ penalties, as well as architectural changes~\citep{huang2024modern}. Its analysis connects this objective-and-regularizer design to local convergence. Recent proposals further improve GAN performance mainly through architecture and training-loop design~\citep{kang2023gigagan,sauer2023styleganT,sauer2022styleganxl}. We retain the real-versus-fake classification setting, but train the discriminator with the BOLT loss and the generator to maximize the resulting BER surrogate. At balanced priors, this leads to an objective within the bounded WGAN/IPM family.

\paragraph{Difference from prior BOLT work.} The ICASSP 2025 BOLT paper~\citep{naeini2025universal} and its expanded companion manuscript~\citep{tavasoli2026bolt} study supervised classification and BER estimation for fixed data distributions. Our setting is different: one of the two class-conditional distributions is the generator distribution, which changes throughout adversarial training. Applying BOLT to GANs therefore requires converting the BER bound into a generator-discriminator saddle-point objective, analyzing the induced discrepancy on the generator distribution, and controlling the discriminator class to address the instability observed with the unconstrained TV-strength objective. This leads to the TV connection in the unconstrained case and the Wasserstein-controlled BOLT-GAN objective in the Lipschitz case.

\paragraph{Notation.}
On a measurable space $(\mathcal{X},\mathcal{F})$, probability measures are denoted by uppercase letters such as $P$, and corresponding densities by lowercase letters such as $p$. For a random variable $X$, $P_X$ denotes its marginal distribution. Expectation with respect to $P$ is denoted by $\mathbb{E}_{X\sim P}[\cdot]$; we also write $\E_P[h]=\E_{X\sim P}[h(X)]$. We use $[n]$ to denote the set $\{1,\dots,n\}$.

\section{Preliminaries}
\label{sec:background}
We consider binary classification with classes $C\in\{1,2\}$ and data space $\mathcal{X}\subseteq\mathbb{R}^d$. Let $q_i:=\Pr(C=i)$ denote the class priors, with $q_1,q_2\in(0,1)$ and $q_2=1-q_1$. Let $P_i$ denote the class-conditional distribution of $X$ given $C=i$, with density $p_i$ w.r.t.\ a common dominating measure $\mu$. The Bayes error rate (BER), denoted by $\varepsilon_{\mathrm{Bayes}}$, is the minimum error probability achievable by any measurable classifier $f:\mathcal{X}\to\{1,2\}$. It is determined from the data distribution as~\citep{bishop2006pattern}
\begin{equation}
\label{eq:bayes_error_def}
\varepsilon_{\mathrm{Bayes}}
= 1-\int_{\mathcal{X}} \max_{i\in\{1,2\}}\{q_i p_i(x)\}\,\mathrm{d}\mu(x),
\end{equation}
and is achieved by the maximum-a-posteriori (MAP) rule
\begin{equation}
\label{eq:map_rule}
\widehat C_{\rm MAP}(x)\in \argmax_{i\in\{1,2\}} q_i p_i(x).
\end{equation}

\subsection{Bayes Optimal Learning Threshold (BOLT)}
With unknown data distribution, it is known that the BER can be approximated from data~\citep{devroye1996probabilistic,fukunaga2013introduction,noshad2019learning,sekeh2020learning}. However, direct high-dimensional BER estimation can incur prohibitive computational complexity. Instead, \citet{naeini2025universal} develop a universal upper bound on the BER of the binary classification problem that can be estimated by sampling an arbitrary bounded BER proxy. The resulting surrogate objective is called the BOLT risk.

The following theorem presents the universal bound for binary classification.
\begin{theorem}[BER bound~\citep{naeini2025universal}]\label{thm:binary_upper_bound}
Let $h:\mathcal{X}\to[0,1]$ be a measurable function. Then, the BER is bounded as\footnote{The result in \citet{naeini2025universal} is given for $f:\mathcal{X}\to[-1,0]$. We represent the result after the reparameterization $h=1+f$.}
\begin{equation}
\label{bound}
\varepsilon_{\mathrm{Bayes}}
\le q_1 + q_2\,\mathbb{E}_{X\sim P_2}[h(X)] - q_1\,\mathbb{E}_{X\sim P_1}[h(X)].
\end{equation}
\end{theorem}

\citet{naeini2025universal} show that the bound in \eqref{bound} is tight when minimized over $h$. This leads to the BOLT framework for classification: replacing $h$ by a parameterized score model $h_\theta$ and minimizing an empirical version of the bound trains $h_\theta$ to estimate a tight BER proxy. We focus on the binary case relevant to GAN discrimination.

\paragraph{BOLT loss.}
Consider a labeled dataset $\mathcal{D}=\{(x_i,C_i): i\in[n]\}$ with $C_i\in\{1,2\}$. BOLT trains the model $h_\theta:\mathcal{X}\to[0,1]$ by minimizing an empirical BOLT risk. For a sample $(x,C)\in\mathcal{D}$, the per-sample loss, referred to as the BOLT loss, is
\begin{equation}
\ell_{\rm BOLT}(h_\theta(x),C)=(-1)^C\,h_\theta(x).
\end{equation}
The bound in Theorem~\ref{thm:binary_upper_bound} can be written in terms of the population risk as
\begin{equation}
\varepsilon_{\mathrm{Bayes}}\le q_1+\mathbb{E}\bigl[\ell_{\rm BOLT}(h_\theta(X),C)\bigr].
\end{equation}
Thus, minimizing the expected BOLT loss trains the score model $h_\theta$ to estimate a tight BER surrogate.

\subsection{Generative Adversarial Networks}

A GAN consists of two models: a generator $g:\mathcal Z\to\mathcal X$ that maps a latent variable $Z\sim P_Z$ to a sample in data space $\mathcal X$, and a discriminator $h:\mathcal X\to[0,1]$ that distinguishes generated samples $g(Z)$ from true data samples $X\sim P_{\mathrm{data}}$. We denote the distribution induced by $g(Z)$ by $P_{\mathrm g}$.

\paragraph*{Vanilla GAN.}
The vanilla GAN~\citep{goodfellow2014generative} treats the discriminator as a probabilistic real-versus-fake decision rule with $h(x)=\Pr(C=\texttt{real}\mid x)$ and uses the cross-entropy loss. The training objective in this case is
\begin{align}
\label{eq:CE}
R_{\rm ML}(g,h) = \mathbb{E}\big[\mathrm{CE}(Y,C)\big],
\end{align}
where $\mathrm{CE}$ denotes cross-entropy, $Y=h(X)$ is the discriminator output for input $X\in\mathcal X$, and $C\in\{1,2\}$ indicates whether $X$ is real ($C=1$) or fake ($C=2$). The label $C$ is drawn from a prior $0<\pi:=\Pr(C=1)<1$, and $X$ is drawn conditionally from $P_1=P_{\mathrm{data}}$ or $P_2=P_{\mathrm g}$\footnote{Typically $\pi=0.5$ is considered.}.
Using $P_2=P_{\mathrm g}$, where $P_{\mathrm g}$ is the distribution of $g(Z)$ for $Z\sim P_Z$, the risk in \eqref{eq:CE} can be expanded as
\begin{equation}
\begin{aligned}
R_{\rm ML}(g,h)
&= -\pi\,\mathbb{E}_{X\sim P_{\rm data}}[\log h(X)]\\
&\quad -(1-\pi)\,\mathbb{E}_{Z\sim P_Z}[\log(1-h(g(Z)))].
\end{aligned}
\label{eq:ML-GAN}
\end{equation}
Here real and generated samples are classes $1$ and $2$, so $\pi=q_1$ and $1-\pi=q_2$. Defining $\mathcal{L}_{\rm ML}(g,h)=-R_{\rm ML}(g,h)$, the generator learns to sample from $P_{\mathrm{data}}$ by solving the following min-max game
\begin{align}
\label{eq:minmax}
\min_g\sup_{h:\mathcal X\to(0,1)}\; \mathcal{L}_{\rm ML}(g,h).
\end{align}
The supremum is over measurable scores; endpoint solutions are understood as limits. For an unrestricted optimal discriminator, the balanced-prior game is proportional, up to an additive constant, to Jensen--Shannon (JS) divergence~\citep[Theorem~1]{goodfellow2014generative}; other priors give its weighted version.

\paragraph*{Wasserstein GAN.}
\citet{arjovsky2017wasserstein} propose a different approach: WGAN uses a real-valued critic that estimates the Wasserstein-$1$ distance between the data and generator distributions, instead of a probabilistic real-versus-fake discriminator. By Kantorovich-Rubinstein duality, the Wasserstein-$1$ distance is expressed as a supremum over $1$-Lipschitz functions. As a result, the WGAN critic is constrained to be $1$-Lipschitz; practical approximations include weight clipping, gradient penalty, and spectral normalization. The resulting min-max objective is
\begin{equation}
\mathcal{L}_{\rm EM}(g,h)=\mathbb{E}_{X\sim P_{\mathrm{data}}}[h(X)]-\mathbb{E}_{Z\sim P_Z}[h(g(Z))],
\end{equation}
and the population game is explicitly
\begin{equation}
\min_g\sup_{\substack{h:\mathcal X\to\mathbb R\\\operatorname{Lip}(h)\le1}}\mathcal L_{\rm EM}(g,h).
\end{equation}
This is the Wasserstein objective~\citep[Section~3, Eq.~(2)]{arjovsky2017wasserstein}. WGAN replaces the JS-based criterion with the weaker Wasserstein distance, which gives the usual intuition for its improved training stability.

\section{Alternative Formulation: GAN with Maximum Discrimination BER}
We begin with the unrestricted BOLT decision score and a density-ratio BER estimator. We then construct the adversarial game and study the effect of restricting the discriminator class.

\subsection{Bayes Optimal Discriminator via BOLT}
Theorem~\ref{bolt-map} gives the optimal decision score for the binary BOLT bound when optimization ranges over all measurable bounded scores.

\begin{theorem}[MAP vs. BOLT optimum]
\label{bolt-map}
Consider the binary classification setup in Section~\ref{sec:background}, and let $P_X$ denote the marginal distribution of $X$. Let
\begin{align}
\label{eq:bolt-pop-risk}
h^\star \in \argmin_{h:\mathcal{X}\to[0,1]}
\; \mathbb{E}\bigl[\ell_{\rm BOLT}(h(X),C)\bigr],
\end{align}
and let $\eta(x):=\Pr(C=1\mid X=x)$ be the posterior probability. Then, for $P_X$-almost every $x$,
\begin{equation}
\label{eq:bolt-pop-minimizer}
h^\star(x)=
\begin{cases}
1, & \eta(x)>0.5,\\
0, & \eta(x)<0.5,\\
\text{any } z\in[0,1], & \text{if }\eta(x)=0.5.
\end{cases}
\end{equation}
Consequently, the plug-in classifier
\begin{equation}
\widehat{C}(x)=
\begin{cases}
1, & h^\star(x)\ge 0.5,\\
2, & h^\star(x)<0.5,
\end{cases}
\end{equation}
is $P_X$-almost everywhere a MAP rule; at posterior ties either label is Bayes-optimal.
\end{theorem}
\begin{proof} See Appendix~\ref{sec:B}.\end{proof}

The following excess-risk bound relates BOLT optimization to classification accuracy for restricted score classes. For any bounded $h$, let $C_h$ threshold it at $1/2$, let $R_B(h)$ be its BOLT risk, and let $R_B^*$ be the unrestricted infimum. Then
\begin{equation}
R_{0\text{-}1}(C_h)-\Bayes\le2\{R_B(h)-R_B^*\}.
\label{eq:calibration-main}
\end{equation}
Indeed, a wrong MAP decision places $h$ at least $1/2$ away from its optimal endpoint, so its conditional excess BOLT risk is at least half the excess classification error. Appendix~\ref{app:calibration} gives the full argument. A restricted class may retain an approximation gap. The ramp example in Appendix~\ref{app:calibration} has BER $0$ and optimal Lipschitz BER bound $1/8$, yet its threshold decision is exact. Bound tightness and decision accuracy are distinct. The endpoint score encodes the classification decision, whereas the BER estimator below uses a separately fitted density ratio.

\begin{corollary}[Hinge form of $\varepsilon_{\mathrm{Bayes}}$]
\label{cor:lr-form}
Define $U(x):=p_1(x)/p_2(x)$ on $\{p_2>0\}$ and $\tau:=q_2/q_1$. On $\{p_2=0<p_1\}$ set $h^\star=1$; null sets are immaterial. Then, an optimizer of the BOLT risk 
is of the form
\begin{align}
\label{eq:hstar-01}
h^\star(x)=
\begin{cases}
0, & U(x)<\tau,\\
1, & U(x)\ge \tau,
\end{cases}
\end{align}
and, defining the hinge function $t_0(u):=[q_2-q_1u]_+$, the BER is given by
\begin{equation}
\label{eq:bolt-tight-and-hinge}
\varepsilon_{\mathrm{Bayes}}
= q_2-\mathbb{E}_{X\sim P_2}\big[t_0(U(X))\big].
\end{equation}
\end{corollary}
\begin{proof}
For $x$ with $p_2(x)>0$, Bayes' rule gives $\eta(x)=q_1U(x)/(q_1U(x)+q_2)$. Hence, $\eta(x)>0.5$ if and only if $U(x)>\tau$, and $\eta(x)<0.5$ if and only if $U(x)<\tau$. Substituting this into Theorem~\ref{bolt-map} gives \eqref{eq:hstar-01}. The hinge representation \eqref{eq:bolt-tight-and-hinge} follows by rewriting the tight BOLT bound at $h^\star$ in terms of $U$.
\end{proof}

For optional post-hoc evaluation with a frozen generator, a plug-in BER estimator requires a separately fitted nonnegative ratio estimate $\widehat U$. For example, an independently fitted posterior model $\widehat\eta\in(0,1)$ gives $\widehat U=(q_2/q_1)\widehat\eta/(1-\widehat\eta)$. GAN class priors are known by design.
\begin{theorem}[Conditional BER estimation]
\label{thm:ber-plug-in}
Fix known priors $q_1,q_2\in(0,1)$ and a nonnegative measurable $\widehat U$ fitted on data $\mathcal T$. Independently draw $X_1,\ldots,X_M\stackrel{\mathrm{i.i.d.}}{\sim}P_2$ and define
\begin{equation}
\widehat\varepsilon_{\rm BOLT}=\frac1M\sum_{j=1}^M\min\{q_1\widehat U(X_j),q_2\}.
\label{eq:ber-estimator-main}
\end{equation}
If $\E_{P_2}|\widehat U-U|\le\varepsilon_0$ conditional on $\mathcal T$, then
\begin{equation}
\begin{aligned}
|\E[\widehat\varepsilon_{\rm BOLT}\mid\mathcal T]-\Bayes|&\le q_1\varepsilon_0,\\
\Var(\widehat\varepsilon_{\rm BOLT}\mid\mathcal T)&\le \frac{q_2^2}{4M}.
\end{aligned}
\label{eq:bolt-bias}
\end{equation}
\end{theorem}
\begin{proof} See Appendix~\ref{app:ber-estimator}.\end{proof}
Conditional on the fitted ratio and frozen generator, increasing $M$ reduces Monte Carlo evaluation variance at the $1/M$ rate. Ratio-approximation error is controlled separately by $\varepsilon_0$. Appendix~\ref{sec:G-ber} validates this estimator using Gaussian distributions.

\subsection{Generation with Maximum Discrimination BER}
\label{sec:max-ber-game}
Theorem~\ref{bolt-map} interprets a BOLT-trained discriminator as an estimator of the Bayes-optimal classifier for the real-versus-fake task. We therefore construct a max-min game, referred to as \emph{Max-BER}, in which the discriminator minimizes the BOLT risk for fixed generator $g$, while the generator maximizes the induced discrimination BER. For a generator $g$ inducing $P_{\mathrm g}$ and a bounded discriminator $h:\mathcal X\to[0,1]$, define
\begin{equation}
\begin{aligned}
\mathcal{L}_{\mathrm{MB}}^{(\pi)}(g,h)
&:= \pi\,\mathbb{E}_{X\sim P_{\mathrm{data}}}[h(X)]\\
&\quad -(1-\pi)\,\mathbb{E}_{Z\sim P_Z}[h(g(Z))],
\end{aligned}
\label{BOLT-GAN-Fun}
\end{equation}
where $\pi=\Pr(C=1)$ is the probability of sampling a real data point and the subscript $\mathrm{MB}$ stands for the Max-BER functional. The BER upper bound is $\pi-\mathcal L_{\mathrm{MB}}^{(\pi)}(g,h)$. Tightening it in $h$ and maximizing it in $g$ therefore yields
\begin{equation}
\min_g\sup_{h:\mathcal X\to[0,1]}\;\mathcal{L}_{\mathrm{MB}}^{(\pi)}(g,h),
\end{equation}
to fool the Bayes-optimal discrimination rule.

\paragraph{Generator distribution.}
The next theorem relates the unrestricted optimized objective to TV, giving an identity at balanced priors. Appendix~\ref{sec:S1-prelims} defines TV and gives its variational form.
\begin{theorem}[Relation between the Max-BER objective and total variation]
\label{thm:bolt-vs-tv}
Fix $\pi\in(0,1)$, and define
\begin{equation}
\label{eq:risk_def_full}
\mathcal{D}^{(\pi)}(g)\triangleq \sup_h \mathcal{L}_{\mathrm{MB}}^{(\pi)}(g,h),
\end{equation}
where the supremum is over all bounded discriminators $h:\mathcal X\to[0,1]$. Then,
\begin{equation}
\label{eq:sum_tv}
\mathcal{D}^{(\pi)}(g)+\mathcal{D}^{(1-\pi)}(g)\ge \mathrm{TV}(P_{\mathrm{data}},P_{\mathrm g}),
\end{equation}
with equality, in the form $2\mathcal{D}^{(0.5)}(g)=\mathrm{TV}(P_{\mathrm{data}},P_{\mathrm g})$, when $\pi=0.5$.
\end{theorem}
\begin{proof}
For any bounded $h$, the two complementary-prior objectives satisfy
\begin{equation}
\begin{aligned}
&\mathcal L_{\mathrm{MB}}^{(\pi)}(g,h)+\mathcal L_{\mathrm{MB}}^{(1-\pi)}(g,h)\\
&\qquad=\mathbb E_{P_{\mathrm{data}}}[h]-\mathbb E_{P_{\mathrm g}}[h].
\end{aligned}
\end{equation}
Taking the supremum over measurable $h:\mathcal X\to[0,1]$ gives TV. The supremum of a sum cannot exceed the sum of the separate suprema, which proves \eqref{eq:sum_tv}. At $\pi=0.5$, the two objectives are identical, so the supremum is exactly $2\mathcal D^{(0.5)}(g)$. Appendix~\ref{app:tv-proof} gives the expanded derivation.
\end{proof}

At $\pi=1/2$, the full unrestricted objective tends to zero exactly when the generator distribution converges in TV. For imbalanced priors, Corollary~\ref{cor:halftv} controls TV through the sum of the two complementary-prior population objectives.

\subsection{Max-BER with Lipschitz Discriminator}
\label{sec:boltgan-framework}
The optimal discriminator in the unconstrained Max-BER problem may be highly non-smooth. This is undesirable in practice because an objective as strong as the TV metric can yield poorly behaved gradients. Motivated by WGAN, we therefore constrain the discriminator to be $1$-Lipschitz and refer to the resulting method as \BOLTGAN.

\paragraph{BOLT-GAN.}
More precisely, \BOLTGAN solves the constrained Max-BER problem
\begin{equation}
\min_g\sup_{h\in\mathcal H_{\mathrm{Lip}}}\;\mathcal{L}_{\mathrm{MB}}^{(\pi)}(g,h),
\end{equation}
where $\mathcal H_{\mathrm{Lip}}$ denotes the class of bounded $1$-Lipschitz discriminators,
\begin{equation}
\mathcal H_{\mathrm{Lip}}=\left\{h:\mathcal X\to[0,1]:
\begin{array}{l}|h(x)-h(y)|\le \|x-y\|,\\ \forall x,y\in\mathcal X\end{array}\right\}.
\end{equation}
We use this formulation to analyze the impact of discriminator continuity on the generator distribution.

\paragraph{BOLT-GAN generator distribution.}
Theorem~\ref{thm:lipschitz-bolt-w1} bounds the BOLT-GAN objective above by Wasserstein-$1$ distance. Appendix~\ref{sec:S1-prelims} gives the distance's definition and variational form.

\begin{theorem}[Wasserstein upper bound]
\label{thm:lipschitz-bolt-w1}
Assume finite first moments, fix $0<\pi\le0.5$, and define
\begin{equation}
\label{eq:lip-bolt}
\mathcal{D}^{(\pi)}_{\mathrm{Lip}}(g)
\triangleq \sup_{h\in\mathcal H_{\mathrm{Lip}}}\mathcal{L}_{\mathrm{MB}}^{(\pi)}(g,h),
\end{equation}
with $\mathcal{L}_{\mathrm{MB}}^{(\pi)}(g,h)$ defined in \eqref{BOLT-GAN-Fun}. Then,
\begin{equation}
\label{eq:sigma-lip}
0\le\mathcal{D}^{(\pi)}_{\mathrm{Lip}}(g)\le \pi W_1(P_{\mathrm{data}},P_{\mathrm g}).
\end{equation}
\end{theorem}
\begin{proof}
For $0<\pi\le0.5$ and $h\ge0$, expanding \eqref{BOLT-GAN-Fun} gives
\begin{equation}
\begin{aligned}
\mathcal L_{\mathrm{MB}}^{(\pi)}(g,h)
&=\pi\bigl(\mathbb E_{P_{\mathrm{data}}}[h]-\mathbb E_{P_{\mathrm g}}[h]\bigr)\\
&\quad-(1-2\pi)\mathbb E_{P_{\mathrm g}}[h]\\
&\le\pi\bigl(\mathbb E_{P_{\mathrm{data}}}[h]-\mathbb E_{P_{\mathrm g}}[h]\bigr).
\end{aligned}
\end{equation}
The last inequality uses $1-2\pi\ge0$. Taking the supremum over $\mathcal H_{\mathrm{Lip}}$ and applying Kantorovich--Rubinstein duality yields $\mathcal D_{\mathrm{Lip}}^{(\pi)}(g)\le\pi W_1(P_{\mathrm{data}},P_{\mathrm g})\le W_1(P_{\mathrm{data}},P_{\mathrm g})$. The lower bound follows from $h\equiv0$. Appendix~\ref{app:w1-proof} gives the same $\pi$-scaled bound in full.
\end{proof}

\paragraph{Converse on a bounded domain.}
For the full balanced class, write $P=P_{\rm data}$, $Q=P_{\rm g}$, and set $D_B=\mathcal D_{\mathrm{Lip}}^{(1/2)}$. On a domain of finite diameter $d_{\mathcal X}$,
\begin{equation}
2D_B\le W_1(P_{\mathrm{data}},P_{\mathrm g})\le2\max\{1,d_{\mathcal X}\}D_B.
\label{eq:bounded-domain-converse}
\end{equation}
To prove the reverse inequality, let $c=\max\{1,d_{\mathcal X}\}$. Every real-valued $1$-Lipschitz potential $f$ has oscillation at most $d_{\mathcal X}$. Hence $h=(f-\inf f)/c$ lies in the bounded class. The additive constant cancels between probability measures, giving $\E_Pf-\E_Qf=c(\E_Ph-\E_Qh)\le2cD_B$; taking the dual supremum proves the bound. Thus the quantities vanish together for the full balanced class on a bounded domain. For images in $[-1,1]^d$, the diameter is $2\sqrt d$, so the converse can be quantitatively loose. Appendix~\ref{app:bounded-domain} gives the full-class population comparison and its dependence on the domain diameter.

\paragraph{Practical score and gradients.}
We use $h_\theta=\sigma(f_\theta)$ and apply interpolation gradient penalty (GP) to the raw score $f_\theta$~\citep{gulrajani2017improved}. GP encourages $\|\nabla_xf_\theta\|_2$ to stay near $1$ on sampled real--fake interpolants, providing local gradient regularization. Under a global $1$-Lipschitz assumption on the raw function, $\operatorname{Lip}(\sigma\circ f)\le1/4$. The supremum over this ideal sigmoid class consequently satisfies $0\le D_\sigma^{(\pi)}\le(\pi/4)W_1$ for $\pi\le1/2$ (Appendix~\ref{app:gradient-tradeoff}).

For fixed critic parameters and $x=g_\varphi(Z)$, the generator update has the chain-rule form
\begin{equation}
\begin{aligned}
\nabla_\varphi\mathcal L_\varphi=-(1-\pi)\E_Z\big[\sigma'(f_\theta(x))\,\\
J_\varphi g_\varphi(Z)^\top\nabla_xf_\theta(x)\big].
\end{aligned}
\label{eq:gradient-main}
\end{equation}
The factor $\sigma'(f)=\sigma(f)(1-\sigma(f))\le1/4$ attenuates generator gradients at extreme logits. Other smooth bounded maps change the gradient scale through their slopes. This scaling also changes the relative strength of regularization, even when the GP coefficient is held fixed.

R1/R2 penalize raw-score gradients toward zero on real/fake samples~\citep{mescheder2018which}. They can be used with the BOLT objective, but impose a different condition from interpolation GP. Comparing these choices requires controlled coefficients and update schedules; Appendix~\ref{app:gradient-tradeoff} gives the formulas.

\begin{figure}[!t]
  \centering
  \begin{minipage}[t]{0.49\linewidth}
    \vspace{0pt}
    \centering
    \begin{minipage}[t][0.205\textheight][t]{\linewidth}
      \centering
      \resizebox{0.96\linewidth}{!}{\begin{tikzpicture}
  \begin{axis}[
     width=3.2in, height=2.2in,
    xlabel={Epoch},
    ylabel={FID ($\downarrow$)},
    xmin=0, xmax=85,
    ymin=20, ymax=170,
    grid=both,
    legend pos=north east,
    legend cell align=left,
    cycle list={{orange!80!black,mark=o},{violet!80!black,mark=s}},
  ]

    \addplot[very thick,orange!80!black,dashed,mark=o,mark size=1.4pt] coordinates {
      (1,152.4028) (5, 89.9976) (10, 66.6347) (15, 61.7712) (20, 58.3934)
      (25, 66.6697) (30, 68.1764) (35, 81.0744) (40, 83.2443) (45, 90.0601)
      (50, 80.4459) (55, 86.7790) (60, 80.4510) (65, 89.3695) (70, 90.2055)
      (75, 93.9522) (80, 95.1438)
    };
    \addlegendentry{WGAN-GP}


    \addplot[very thick,red!75!black,solid,mark=square*,mark size=1.4pt] coordinates {
      (1,109.8975) (5, 58.2135) (10, 53.6458) (15, 48.2462) (20, 45.9401)
      (25, 42.7842) (30, 39.4595) (35, 39.5993) (40, 39.2208) (45, 40.7992)
      (50, 36.8481) (55, 36.7657) (60, 33.7018) (65, 34.4446) (70, 36.8438)
      (75, 34.7257) (80, 34.4183)
    };
    \addlegendentry{BOLT-GAN}
  \end{axis}
\end{tikzpicture}}
    \end{minipage}
\captionof{figure}{CIFAR-10 FID ($\downarrow$) vs. epochs for \WGANGP{} and \BOLTGAN{} with $\lambda_{\rm GP}=10$. Curves show mean FID over three seeds.}
    \label{fig:fid_curve}
  \end{minipage}\hfill
  \begin{minipage}[t]{0.49\linewidth}
    \vspace{0pt}
    \centering
    \begin{minipage}[t][0.205\textheight][t]{\linewidth}
      \centering
      \begin{minipage}[t]{0.49\linewidth}
        \vspace{0pt}
        \centering
        \includegraphics[width=\linewidth,height=0.185\textheight,keepaspectratio]{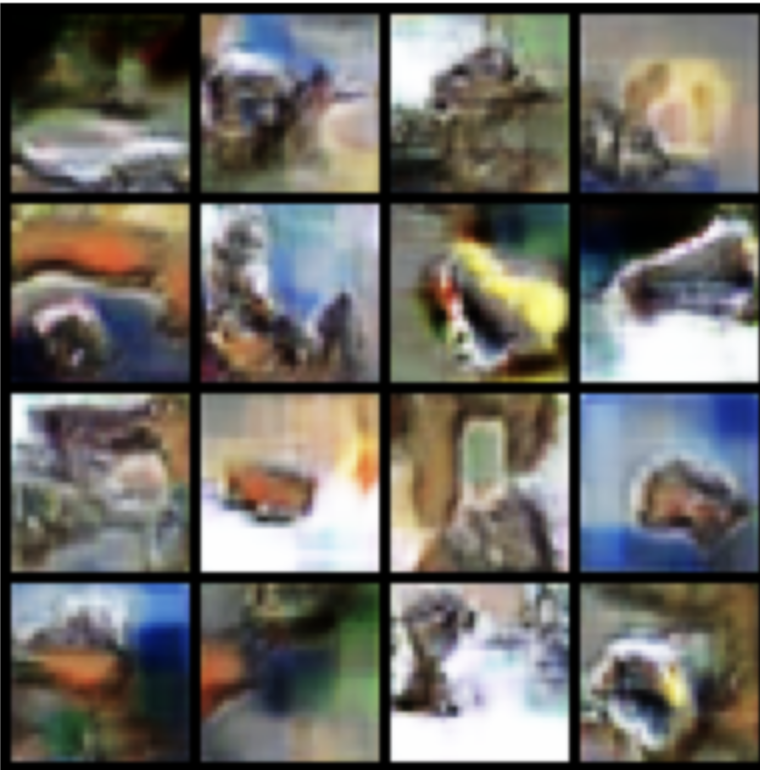}\par
        {\scriptsize \WGANGP}
      \end{minipage}\hfill
      \begin{minipage}[t]{0.49\linewidth}
        \vspace{0pt}
        \centering
        \includegraphics[width=\linewidth,height=0.185\textheight,keepaspectratio]{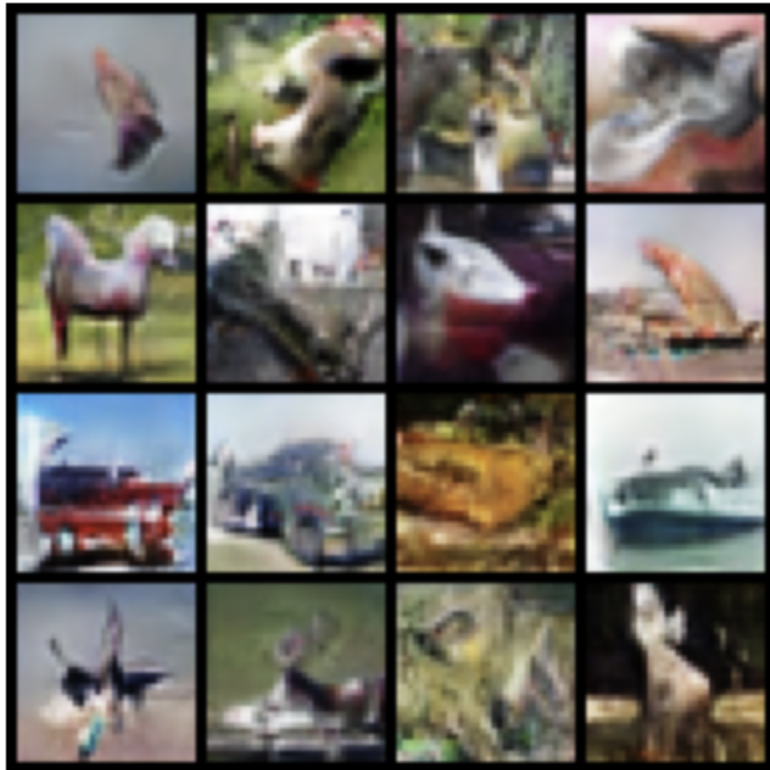}\par
        {\scriptsize \BOLTGAN}
      \end{minipage}
    \end{minipage}
    \captionof{figure}{CIFAR-10 samples generated after $100$ epochs of training by \WGANGP and \BOLTGAN with the identical architecture and optimization setting.}
    \label{fig:cifar100_compare}
  \end{minipage}
\end{figure}

\begin{table}[!t]
\centering
\caption{Discriminator-continuity and hyperparameter comparisons, separated by purpose. Lower FID is better. Panel (a) uses the same $20$-epoch budget; panels (b,c) use CIFAR-10. Panel (b) reports mean$\pm$SD over three seeds. The default $\lambda_{\rm GP}=10$ is retained in Table~\ref{tab:fid20}, independently of the sweep optimum.}
\label{tab:stability-ablation}
\MainTableSetup
\begin{tabularx}{\linewidth}{@{}Ycc@{}}
\toprule
\multicolumn{3}{l}{\textbf{(a) Unconstrained versus GP-regularized BOLT}}\\
Dataset & Max-BER FID & BOLT-GAN FID\\
\midrule
CIFAR-10 & $>300$ & $\mathbf{44.2}$\\
CelebA-64 & $>300$ & $\mathbf{9.2}$\\
LSUN Bedroom-64 & $>400$ & $\mathbf{28.4}$\\
\midrule
\multicolumn{3}{l}{\textbf{(b) Real/fake prior sensitivity}}\\
Prior $\pi$ & FID at $20$ epochs & FID at $100$ epochs\\
\midrule
$0.50$ & $\mathbf{44.2\pm0.6}$ & $36.6\pm0.8$\\
$0.45$ & $45.3\pm0.7$ & $\mathbf{36.2\pm0.7}$\\
$0.40$ & $48.5\pm1.0$ & $41.0\pm1.1$\\
$0.35$ & $51.4\pm1.3$ & $44.8\pm1.4$\\
\midrule
\multicolumn{3}{l}{\textbf{(c) GP-coefficient sweep at a fixed budget}}\\
Method & FID for $\lambda=1/5/10/20$ & Best FID\\
\midrule
\WGANGP & $60.8/59.6/60.0/61.1$ & $\mathbf{59.6}$\\
\BOLTGAN & $44.3/43.6/44.2/46.7$ & $\mathbf{43.6}$\\
\bottomrule
\end{tabularx}
\end{table}

\begin{table}[!t]
\centering
\caption{Fixed-budget comparisons at $20$ epochs. Panel (a) holds GP and its coefficient $\lambda_{\rm GP}=10$ fixed. Precision (P) and Recall (R) use $n=2000$, $k=3$. Panel (b) presents HINGE-GAN with spectral normalization (SN). Entries are point estimates.}
\label{tab:fid20}
\MainTableSetup
\begin{tabularx}{\linewidth}{@{}Yrrrrrr@{}}
\toprule
\multicolumn{7}{l}{\textbf{(a) Common architecture, optimizer, schedule, and GP}}\\
& \multicolumn{2}{c}{FID ($\downarrow$)} & \multicolumn{2}{c}{P ($\uparrow$)} & \multicolumn{2}{c}{R ($\uparrow$)}\\
\cmidrule(lr){2-3}\cmidrule(lr){4-5}\cmidrule(lr){6-7}
Dataset & WGAN-GP & BOLT-GAN & WGAN-GP & BOLT-GAN & WGAN-GP & BOLT-GAN\\
\midrule
CIFAR-10 & $60.0$ & $\mathbf{44.2}$ & $.714$ & $\mathbf{.763}$ & $.183$ & $\mathbf{.352}$\\
LSUN Church-64 & $43.5$ & $\mathbf{14.8}$ & $.681$ & $\mathbf{.750}$ & $.176$ & $\mathbf{.411}$\\
LSUN Bedroom-64 & $102.5$ & $\mathbf{28.4}$ & $.388$ & $\mathbf{.598}$ & $.098$ & $\mathbf{.450}$\\
CelebA-64 & $10.3$ & $\mathbf{9.2}$ & $.672$ & $\mathbf{.840}$ & $.104$ & $\mathbf{.548}$\\
\bottomrule
\end{tabularx}
\smallskip
\begin{tabularx}{\linewidth}{@{}Yrrrr@{}}
\toprule
\multicolumn{5}{l}{\textbf{(b) HINGE-GAN with spectral normalization}}\\
Dataset & CIFAR-10 & Church-64 & Bedroom-64 & CelebA-64\\
FID ($\downarrow$) & $46.9$ & $19.7$ & $28.8$ & $19.6$\\
\bottomrule
\end{tabularx}
\end{table}

\begin{table}[!t]
\centering
\caption{Additional experiments separated by purpose. Arrows specify the preferred metric direction. Entries are point estimates unless marked otherwise. The conditional results use the final checkpoint at epoch $50$; the best checkpoints are reported separately in Appendix~\ref{sec:G-cond}.}
\label{tab:stronger-and-long}
\MainTableSetup
\begin{tabularx}{\linewidth}{@{}Ylcc@{}}
\toprule
\multicolumn{4}{l}{\textbf{(a) Discriminator changes: FID ($\downarrow$)}}\\
Dataset & Discriminator & \WGANGP & \BOLTGAN\\
\midrule
Tiny ImageNet & Residual CNN & $74.8$ & $\mathbf{68.9}$\\
Tiny ImageNet & ViT-style & $70.6$ & $\mathbf{64.1}$\\
STL-10 & Residual CNN & $49.7$ & $\mathbf{44.8}$\\
STL-10 & ViT-style & $45.5$ & $\mathbf{40.9}$\\
\bottomrule
\end{tabularx}
\smallskip
\begin{tabularx}{\linewidth}{@{}Ylccc@{}}
\toprule
\multicolumn{5}{l}{\textbf{(b) Image quality and longer training}}\\
Dataset / budget & Metric & \WGANGP & \HingeSN & \BOLTGAN\\
\midrule
ImageNet-32 / $30$ ep & FID ($\downarrow$) & $58.0$ & $42.3$ & $\mathbf{34.8}$\\
CIFAR-10 / $100$ ep & FID@50k ($\downarrow$) & $68.5$ & $35.8$ & $\mathbf{34.9}$\\
 & KID$\times10^3$ ($\downarrow$) & $48.7$ & $24.2$ & $\mathbf{23.1}$\\
 & IS ($\uparrow$) & $5.71$ & $7.32$ & $\mathbf{7.41}$\\
Tiny ImageNet / $100$ ep & FID@50k ($\downarrow$) & $74.8$ & $71.6$ & $\mathbf{68.9}$\\
STL-10 / $100$ ep & FID@50k ($\downarrow$) & $49.7$ & $46.9$ & $\mathbf{44.8}$\\
\bottomrule
\end{tabularx}
\smallskip
\begin{tabularx}{\linewidth}{@{}Yccc@{}}
\toprule
\multicolumn{4}{l}{\textbf{(c) Mode coverage and conditional generation}}\\
Setting / metric & \WGANGP & Other baseline & \BOLTGAN\\
\midrule
25-Gaussians / modes recovered ($\uparrow$) & $\mathbf{25.0\pm0.0}$ & NS-GAN: $6.0\pm1.4$ & $\mathbf{25.0\pm0.0}$\\
Conditional CIFAR-10 / FID@10k ($\downarrow$) & unstable ($>300$) & \HingeSN: $63.2$ & $\mathbf{44.6}$\\
\bottomrule
\end{tabularx}
\end{table}

\paragraph{Mini-batch updates.}
Write $g_\varphi$ for the generator and let $\widehat{\mathcal L}_{\mathrm{MB}}^{(\pi)}$ be the sample average of \eqref{BOLT-GAN-Fun} over $B$ real samples and $B$ independent latent draws. With ${\rm GP}$ denoting the sampled penalty including its coefficient, training minimizes
\begin{equation}
\label{eq:batch-training-losses}
\begin{aligned}
\mathcal L_\theta
  &=-\widehat{\mathcal L}_{\mathrm{MB}}^{(\pi)}(g_\varphi,h_\theta)+{\rm GP},\\
\mathcal L_\varphi
  &=-\frac{1-\pi}{B}\sum_{j=1}^B h_\theta\bigl(g_\varphi(z'_j)\bigr).
\end{aligned}
\end{equation}
The critic update fixes the generator; the generator update uses a fresh latent batch and differentiates through the fixed critic. Its real-data term is constant and omitted. GP acts only on the raw critic. Defaults are $B=64$, $\pi=0.5$, and five critic updates per generator update; Appendix~\ref{sec:H} gives the algorithm.

\section{Experiments and Results}
\label{sec:experiments}
Our experiments examine discriminator continuity, fixed-budget objective comparisons, and extensions to other architectures and datasets.

\paragraph{Experimental setup.}
The core benchmarks are CIFAR-10, LSUN Church-64, LSUN Bedroom-64, and CelebA-64. The main comparison uses a residual DCGAN backbone~\citep{gulrajani2017improved}; BOLT-GAN and WGAN-GP share the architecture, optimizer, schedule, batch size, and raw-score GP coefficient $10$. GP regularizes raw-score gradients on sampled interpolants. HINGE-GAN uses SN and is presented separately in Table~\ref{tab:fid20}. The extensions include ImageNet-32, Tiny ImageNet, and STL-10. Dataset sources and implementation details are given in Appendix~\ref{sec:H}.

FID uses \texttt{pytorch-fid}~\citep{seitzer2020pytorch} with $10$k generated samples unless marked FID@50k. The common $20$-epoch budget evaluates early training performance, and longer runs are presented in Table~\ref{tab:stronger-and-long}. CIFAR-10 also includes KID~\citep{binkowski2018demystifying} and IS~\citep{salimans2016improved}. Improved Precision/Recall~\citep{kynkaanniemi2019improved} uses $(n=\PRn,k=\PRk)$. The main comparisons report point estimates, with three-seed mean and standard deviation given for the prior sweep in Table~\ref{tab:stability-ablation}.

\subsection{Effect of Discriminator Continuity}
Table~\ref{tab:stability-ablation} shows that adding GP improves the unregularized BOLT runs under the same $20$-epoch budget. FID decreases from above $300$ to $44.2$ on CIFAR-10 and to $9.2$ on CelebA-64, and from above $400$ to $28.4$ on LSUN Bedroom-64. The prior $\pi$ controls the real/fake sampling balance. At $100$ epochs, the FID difference between $\pi=0.45$ and $0.50$ is $0.4$, smaller than the run-to-run standard deviations. Across $\lambda_{\rm GP}\in\{1,5,10,20\}$, BOLT-GAN's FID ranges from $43.6$ to $46.7$, compared with $59.6$ to $61.1$ for WGAN-GP. Both methods reach their lowest FID in this sweep at $\lambda_{\rm GP}=5$, while the main comparison keeps the common default $10$. Appendix~\ref{sec:G-lip} varies the regularizer with the BOLT objective fixed.

\subsection{BOLT-GAN versus Baselines}
Table~\ref{tab:fid20} reports lower FID and higher Precision/Recall for BOLT-GAN than WGAN-GP in the four tested settings. The largest absolute FID reduction is on LSUN Bedroom-64, from $102.5$ to $28.4$, accompanied by increases in precision from $0.388$ to $0.598$ and recall from $0.098$ to $0.450$. On LSUN Church-64, FID decreases from $43.5$ to $14.8$ and recall increases from $0.176$ to $0.411$. These comparisons combine lower FID with improvements in both measured precision and recall. The matched configuration compares the objectives at a common raw-score GP coefficient; the sigmoid parameterization changes the relative scale of the adversarial term and the penalty. Figure~\ref{fig:fid_curve} shows their training trajectories under this configuration. HINGE-GAN/SN is presented separately to distinguish the change in regularization.

\subsection{Additional Experiments}
\paragraph{Discriminator variants.}
Table~\ref{tab:stronger-and-long}(a) reports lower FID for BOLT-GAN with both residual-CNN and ViT-style critics. On Tiny ImageNet, the residual-CNN comparison gives FID $68.9$ for BOLT-GAN and $74.8$ for WGAN-GP; with the ViT-style critic, the values are $64.1$ and $70.6$. On STL-10, the corresponding comparisons are $44.8$ versus $49.7$ and $40.9$ versus $45.5$. The lower BOLT-GAN FID is observed within each of the four dataset--critic configurations.

\paragraph{Longer training and conditional generation.}
Table~\ref{tab:stronger-and-long}(b) extends the evaluation to $100$ epochs. On CIFAR-10, BOLT-GAN attains FID@50k $34.9$, KID$\times10^3$ $23.1$, and IS $7.41$, compared with $35.8$, $24.2$, and $7.32$ for HINGE-GAN/SN. The separate three-seed FID@10k summary in Table~\ref{tab:stability-ablation} is $36.6\pm0.8$. ImageNet-32 at $30$ epochs gives FID $34.8$ for BOLT-GAN, $42.3$ for HINGE-GAN/SN, and $58.0$ for WGAN-GP. For conditional CIFAR-10, the final epoch-$50$ checkpoint gives FID@10k $44.6$ for BOLT-GAN and $63.2$ for HINGE-GAN/SN; Appendix~\ref{sec:G-cond} gives the training curves and best-checkpoint results.

\paragraph{Mode coverage and BER estimation.}
On the synthetic $25$-Gaussian mixture, a mode is counted when at least one generated sample lies within radius $0.15$ of its center. Under this criterion, BOLT-GAN and WGAN-GP cover all $25$ modes in each of three seeds, while NS-GAN covers $6.0\pm1.4$. Appendix~\ref{sec:G-25g} presents the coverage trajectories and finite-grid histogram discrepancy. The Gaussian study in Appendix~\ref{sec:G-ber} evaluates the conditional BER estimator using oracle, biased, and fitted density ratios. It uses $16$-dimensional distributions with true BER $0.1586553$ and evaluation sizes $M\in\{128,512,2048,8192\}$. The fitted ratio uses $512$ training samples per class. With each ratio held fixed, $4{,}000$ repetitions at each evaluation size give empirical-to-exact variance ratios between $0.972$ and $1.038$. This agreement supports the conditional $1/M$ variance scaling in Theorem~\ref{thm:ber-plug-in} and illustrates the separate roles of ratio approximation and evaluation sample size.

\section{Conclusion}
BOLT gives a BER-based adversarial game whose unrestricted value at balanced priors is half the TV distance. A bounded Lipschitz restriction yields a WGAN-style IPM with explicit Wasserstein bounds. With sigmoid scores and GP, BOLT-GAN improves on WGAN-GP in the matched image-generation experiments. The theoretical analysis and empirical results connect Bayes-error-based discrimination with adversarial generative modeling.
\label{aistats:mainend}
\FloatBarrier
\bibliographystyle{plainnat}
\bibliography{references}

@misc{seitzer2020pytorch,
  title={{pytorch-fid}: {FID} Score for {PyTorch}},
  author={Seitzer, Maximilian},
  year={2020},
  note={\url{https://github.com/mseitzer/pytorch-fid}},
}

@inproceedings{goodfellow2014generative,
  title        = {{Generative} {Adversarial} {Nets}},
  author       = {Goodfellow, Ian J. and Pouget-Abadie, Jean and Mirza, Mehdi and Xu, Bing and Warde-Farley, David and Ozair, Sherjil and Courville, Aaron and Bengio, Yoshua},
  booktitle    = {Advances in Neural Information Processing Systems (NeurIPS)},
  volume       = {27},
  year         = {2014}
}

@inproceedings{arjovsky2017wasserstein,
  title={{Wasserstein} {Generative} {Adversarial} {Networks}},
  author={Arjovsky, Martin and Chintala, Soumith and Bottou, L{\'e}on},
  booktitle={International Conference on Machine Learning (ICML)},
  pages={214--223},
  year={2017},
  organization={PMLR}
}

@inproceedings{nowozin2016f,
  title={{f-GAN}: Training generative neural samplers using variational divergence minimization},
  author={Nowozin, Sebastian and Cseke, Botond and Tomioka, Ryota},
  booktitle={Advances in Neural Information Processing Systems (NeurIPS)},
  volume={29},
  year={2016}
}

@inproceedings{gulrajani2017improved,
  title={Improved training of {Wasserstein} {GAN}s},
  author={Gulrajani, Ishaan and Ahmed, Faruk and Arjovsky, Martin and Dumoulin, Vincent and Courville, Aaron},
  booktitle={Advances in Neural Information Processing Systems (NeurIPS)},
  volume={30},
  year={2017}
}

@inproceedings{miyato2018spectral,
  title={Spectral normalization for generative adversarial networks},
  author={Miyato, Takeru and Kataoka, Toshiki and Koyama, Masanori and Yoshida, Yuichi},
  booktitle={International Conference on Learning Representations (ICLR)},
  year={2018}
}

@book{bishop2006pattern,
  title={Pattern recognition and machine learning},
  author={Bishop, Christopher M},
  year={2006},
  publisher={Springer}
}

@book{devroye1996probabilistic,
  title={A {Probabilistic} {Theory} of {Pattern} {Recognition}},
  author={Devroye, Luc and Gy{\"o}rfi, L{\'a}szl{\'o} and Lugosi, G{\'a}bor},
  year={1996},
  publisher={Springer Science \& Business Media}
}

@article{noshad2019learning,
  title={Learning to benchmark: Determining best achievable misclassification error from training data},
  author={Noshad, Morteza and Xu, Li and Hero, Alfred},
  journal={arXiv preprint arXiv:1909.07192},
  year={2019}
}

@inproceedings{ishida2023performance,
  title={Is the Performance of My Deep Network Too Good to Be True? {A} Direct Approach to Estimating the {Bayes} Error in Binary Classification},
  author={Ishida, Takashi and Yamane, Ikko and Charoenphakdee, Nontawat and Niu, Gang and Sugiyama, Masashi},
  booktitle={International Conference on Learning Representations (ICLR)},
  year={2023}
}

@inproceedings{heusel2017gans,
  title={{GANs} trained by a two time-scale update rule converge to a local {Nash} equilibrium},
  author={Heusel, Martin and Ramsauer, Hubert and Unterthiner, Thomas and Nessler, Bernhard and Hochreiter, Sepp},
  booktitle={Advances in Neural Information Processing Systems (NeurIPS)},
  volume={30},
  year={2017}
}

@inproceedings{mao2017least,
  title={Least Squares Generative Adversarial Networks},
  author={Mao, Xudong and Li, Qing and Xie, Haoran and Lau, Raymond YK and Wang, Zhen and Smolley, Stephen Paul},
  booktitle={IEEE International Conference on Computer Vision (ICCV)},
  pages={2794--2802},
  year={2017}
}

@article{berthelot2017began,
  title={{BEGAN}: Boundary equilibrium generative adversarial networks},
  author={Berthelot, David and Schumm, Thomas and Metz, Luke},
  year={2017},
  journal      = {arXiv preprint arXiv:1703.10717}
}

@book{fukunaga2013introduction,
  title={Introduction to Statistical Pattern Recognition},
  author={Fukunaga, Keinosuke},
  year={1990},
  publisher={Academic Press},
  edition      = {Second}
}

@book{villani2009optimal,
  title={{Optimal Transport: Old and New}},
  author={Villani, C{\'e}dric},
  volume={338},
  year={2009},
  publisher={Springer},
  address={Berlin, Germany},
  series       = {Grundlehren der mathematischen Wissenschaften},
  doi          = {10.1007/978-3-540-71050-9}
}

@inproceedings{jolicoeur2019relativistic,
  title     = {The relativistic discriminator: A key element missing from standard {GAN}},
  author    = {Jolicoeur-Martineau, Alexia},
  booktitle = {International Conference on Learning Representations},
  year      = {2019}
}

@article{sekeh2020learning,
  title={Learning to Bound the Multi-Class {Bayes} Error},
  author={Sekeh, Salimeh Yasaei and Oselio, Brandon and Hero, Alfred O.},
  journal={IEEE Transactions on Signal Processing},
  volume={68},
  pages={3793--3807},
  year={2020},
  month={May}
}

@inproceedings{huang2024modern,
  title     = {The {GAN} is Dead; Long Live the {GAN}! {A} Modern {GAN} Baseline},
  author    = {Huang, Yiwen and Gokaslan, Aaron and Kuleshov, Volodymyr and Tompkin, James},
  booktitle = {Advances in Neural Information Processing Systems},
  volume    = {37},
  year      = {2024}
}

@inproceedings{noshad2018rate,
  title={Rate-Optimal Meta Learning of Classification Error},
  author={Noshad, Morteza and Hero, Alfred O.},
  booktitle={IEEE International Conference on Acoustics, Speech and Signal Processing (ICASSP)},
  pages={2481--2485},
  year={2018},
  doi          = {10.1109/ICASSP.2018.8461949}
}

@inproceedings{petzka2018lipsgan,
  title        = {{On the regularization of Wasserstein GANs}},
  author       = {Petzka, Henning and Fischer, Asja and Lukovnikov, Denis},
  booktitle    = {International Conference on Learning Representations (ICLR)},
  year         = {2018}
}

@inproceedings{mescheder2018which,
  title        = {Which Training Methods for {GANs} Do Actually Converge?},
  author       = {Mescheder, Lars and Geiger, Andreas and Nowozin, Sebastian},
  booktitle    = {International Conference on Machine Learning (ICML)},
  pages        = {3481--3490},
  year         = {2018},
  organization = {PMLR}
}

@article{kodali2017dragan,
  title        = {On Convergence and Stability of {GANs}},
  author       = {Kodali, Naveen and Abernethy, Jacob and Hays, James and Kira, Zsolt},
  year         = {2017},
  journal      = {arXiv preprint arXiv:1705.07215}
}

@inproceedings{cisse2017parseval,
  title        = {{Parseval} Networks: Improving Robustness to Adversarial Examples},
  author       = {Ciss{\'e}, Moustapha and Bojanowski, Piotr and Grave, Edouard and Dauphin, Yann and Usunier, Nicolas},
  booktitle    = {International Conference on Machine Learning (ICML)},
  pages        = {854--863},
  year         = {2017},
  organization = {PMLR}
}

@inproceedings{anil2019sorting,
  title        = {Sorting Out {Lipschitz} Function Approximation},
  author       = {Anil, Cem and Lucas, James and Grosse, Roger},
  booktitle    = {International Conference on Machine Learning (ICML)},
  pages        = {291--301},
  year         = {2019},
  organization = {PMLR}
}

@book{dudley2002real,
  title     = {Real Analysis and Probability},
  author    = {Dudley, Richard M.},
  publisher = {Cambridge University Press},
  year      = {2002}
}

@inproceedings{naeini2025universal,
  title     = {Universal Training of Neural Networks to Achieve {Bayes} Optimal Classification Accuracy},
  author    = {Tavasoli Naeini, Mohammadreza and Bereyhi, Ali and Noshad, Morteza and Liang, Ben and Hero, Alfred O.},
  booktitle = {IEEE International Conference on Acoustics, Speech and Signal Processing (ICASSP)},
  pages     = {1--5},
  year      = {2025},
  publisher = {IEEE},
  doi          = {10.1109/ICASSP49660.2025.10890777}
}

@misc{tavasoli2026bolt,
  title        = {{BOLT}: Universal Classifier Training via a {Bayes}-Error Upper Bound},
  author       = {Tavasoli Naeini, Mohammadreza and Bereyhi, Ali and Noshad, Morteza and Liang, Ben and Hero, Alfred O.},
  year         = {2025},
  howpublished = {\href{https://arxiv.org/abs/2501.07754}{arXiv:2501.07754}},
  eprint       = {2501.07754},
  archivePrefix = {arXiv}
}

@inproceedings{kynkaanniemi2019improved,
  title     = {Improved Precision and Recall Metric for Assessing Generative Models},
  author    = {Kynk{\"a}{\"a}nniemi, Tuomas and Karras, Tero and Laine, Samuli and Lehtinen, Jaakko and Aila, Timo},
  booktitle = {Advances in Neural Information Processing Systems},
  year      = {2019}
}

@inproceedings{kang2023gigagan,
  title     = {Scaling up {GANs} for Text-to-Image Synthesis},
  author    = {Kang, Minguk and Zhu, Jun-Yan and Zhang, Richard and Park, Jaesik and Shechtman, Eli and Paris, Sylvain and Park, Taesung},
  booktitle = {IEEE/CVF Conference on Computer Vision and Pattern Recognition (CVPR)},
  year      = {2023}
}

@inproceedings{sauer2023styleganT,
  title     = {{StyleGAN-T}: Unlocking the Power of {GANs} for Fast Large-Scale Text-to-Image Synthesis},
  author    = {Sauer, Axel and Karras, Tero and Laine, Samuli and Geiger, Andreas and Aila, Timo},
  booktitle = {International Conference on Machine Learning (ICML)},
  year      = {2023}
}

@inproceedings{sauer2022styleganxl,
  title     = {{StyleGAN-XL}: Scaling {StyleGAN} to Large Diverse Datasets},
  author    = {Sauer, Axel and Schwarz, Katja and Geiger, Andreas},
  booktitle = {ACM SIGGRAPH 2022 Conference Proceedings},
  year      = {2022}
}

@article{moon2018information,
  title = {Ensemble Estimation of Information Divergence},
  author = {Moon, Kevin R. and Sricharan, Kumar and Greenewald, Kristjan and Hero, III, Alfred O.},
  journal = {Entropy},
  volume = {20},
  number = {8},
  pages = {560},
  year = {2018},
  doi = {10.3390/e20080560},
  url = {https://doi.org/10.3390/e20080560}
}

@inproceedings{noshad2017direct,
  title = {Direct Estimation of Information Divergence Using Nearest Neighbor Ratios},
  author = {Noshad, Morteza and Moon, Kevin R. and Yasaei Sekeh, Salimeh and Hero, III, Alfred O.},
  booktitle = {2017 IEEE International Symposium on Information Theory (ISIT)},
  pages = {903--907},
  year = {2017},
  publisher = {IEEE},
  url = {https://arxiv.org/abs/1702.05222}
}

@article{nguyen2010estimating,
  title = {Estimating Divergence Functionals and the Likelihood Ratio by Convex Risk Minimization},
  author = {Nguyen, XuanLong and Wainwright, Martin J. and Jordan, Michael I.},
  journal = {IEEE Transactions on Information Theory},
  volume = {56},
  number = {11},
  pages = {5847--5861},
  year = {2010},
  doi = {10.1109/TIT.2010.2068870},
  url = {https://doi.org/10.1109/TIT.2010.2068870}
}

@inproceedings{miyato2018projection,
  title={{cGANs} with Projection Discriminator},
  author={Takeru Miyato and Masanori Koyama},
  booktitle={International Conference on Learning Representations},
  year={2018},
  url={https://arxiv.org/abs/1802.05637}
}

@inproceedings{binkowski2018demystifying,
  author = {Bi{\'n}kowski, Miko{\l}aj and Sutherland, Danica J. and Arbel, Michael and Gretton, Arthur},
  title = {Demystifying {MMD} {GANs}},
  booktitle = {International Conference on Learning Representations (ICLR)},
  year = {2018},
  url = {https://openreview.net/forum?id=r1lUOzWCW}
}

@inproceedings{salimans2016improved,
  author = {Salimans, Tim and Goodfellow, Ian and Zaremba, Wojciech and Cheung, Vicki and Radford, Alec and Chen, Xi},
  title = {Improved Techniques for Training {GANs}},
  booktitle = {Advances in Neural Information Processing Systems (NeurIPS)},
  volume = {29},
  year = {2016}
}

@techreport{krizhevsky2009learning,
  author = {Krizhevsky, Alex},
  title = {Learning Multiple Layers of Features from Tiny Images},
  institution = {University of Toronto},
  year = {2009},
  url = {https://www.cs.toronto.edu/~kriz/learning-features-2009-TR.pdf}
}

@article{yu2015lsun,
  author = {Yu, Fisher and Zhang, Yinda and Song, Shuran and Seff, Ari and Xiao, Jianxiong},
  title = {{LSUN}: Construction of a Large-Scale Image Dataset using Deep Learning with Humans in the Loop},
  journal = {arXiv preprint arXiv:1506.03365v2},
  year = {2015}
}

@inproceedings{liu2015face,
  author = {Liu, Ziwei and Luo, Ping and Wang, Xiaogang and Tang, Xiaoou},
  title = {Deep Learning Face Attributes in the Wild},
  booktitle = {IEEE International Conference on Computer Vision (ICCV)},
  pages = {3730--3738},
  year = {2015}
}

@inproceedings{deng2009imagenet,
  author = {Deng, Jia and Dong, Wei and Socher, Richard and Li, Li-Jia and Li, Kai and Fei-Fei, Li},
  title = {{ImageNet}: A Large-Scale Hierarchical Image Database},
  booktitle = {IEEE Conference on Computer Vision and Pattern Recognition (CVPR)},
  pages = {248--255},
  year = {2009}
}

@misc{stanford2015tinyimagenet,
  author = {{Stanford University CS231n}},
  title = {{Tiny ImageNet} Challenge},
  howpublished = {Course dataset},
  year = {2015},
  url = {https://cs231n.stanford.edu/2015/project.html}
}

@inproceedings{coates2011analysis,
  author = {Coates, Adam and Ng, Andrew and Lee, Honglak},
  title = {An Analysis of Single-Layer Networks in Unsupervised Feature Learning},
  booktitle = {International Conference on Artificial Intelligence and Statistics (AISTATS)},
  volume = {15},
  pages = {215--223},
  publisher = {PMLR},
  year = {2011}
}

\clearpage
\appendix

This supplementary material accompanies the main paper. Appendix~A gives definitions and preliminaries used throughout. Appendices~B-D contain proofs of the main results: Appendix~\ref{sec:B} proves Theorem~\ref{bolt-map} (BOLT recovers the MAP classifier); Appendix~\ref{sec:C} proves Corollary~\ref{cor:lr-form} (hinge form of $\varepsilon_{\mathrm{Bayes}}$); Appendix~\ref{sec:D} proves Theorem~\ref{thm:ber-plug-in} (bias and variance of the BER estimator). Appendix~\ref{sec:E} gives further results on BOLT-GAN. Appendix~\ref{sec:F} provides details on Lipschitz enforcement. Appendices~\ref{sec:G} and~\ref{sec:H} contain additional experimental results and implementation and evaluation details, respectively.

\section{Definitions and Preliminaries}
\label{sec:S1-prelims}

\subsection{Metrics}
We recall the total variation (TV) distance and the Wasserstein-\(p\) distance, together with basic equivalences
used in our analysis.

\begin{definition}[Total variation]
Let \(P\) and \(Q\) be probability measures on a measurable space \((\mathcal{X},\mathcal{F})\).
The total variation distance is
\begin{align}
\mathrm{TV}(P,Q)\;\defeq\;\sup_{A\in\mathcal{F}}\big|P(A)-Q(A)\big|.
\end{align}
\end{definition}
The total variation can be computed by equivalent forms. The following lemma gives three equivalent expressions for it. 
\begin{lemma}[Equivalent forms of total variation \citep{dudley2002real}]
\label{lem:tv-forms}
For probability measures \(P\) and \(Q\)  on \((\mathcal{X},\mathcal{F})\),
\begin{align}
\mathrm{TV}(P,Q)
&= \sup_{f:\,\mathcal{X}\to[0,1]}\Big(\E_{X\sim P}[f(X)]-\E_{X\sim Q}[f(X)]\Big)
\label{eq:tv-sup-01}
\\
&= \tfrac12 \sup_{\|g\|_\infty\le 1}\Big|\E_{X\sim P}[g(X)]-\E_{X\sim Q}[g(X)]\Big|, \label{eq:tv-sup-bounded}
\end{align}
and if \(P\) and \(Q\) are absolutely continuous w.r.t.\ a \(\sigma\)-finite measure \(\mu\) with densities \(p\) and \(q\),
\begin{align}
\mathrm{TV}(P,Q) \;=\; \tfrac12\int_{\mathcal{X}} |p(x)-q(x)|\,\mathrm{d}\mu(x).
\label{eq:tv-l1}
\end{align}
\end{lemma}
Next, we define the Wasserstein distance. 
\begin{definition}[Wasserstein-\(p\) distance]
\label{def:wp}
Let \((\mathcal{X},d)\) be a metric space and \(p\in[1,\infty)\).
For probability measures \(P\) and \(Q\) on \(\mathcal{X}\) with finite \(p\)-th moments, the Wasserstein-\(p\) distance is
\begin{align}
W_p(P,Q)
\;\defeq\;
\Bigg(\inf_{\gamma\in\Gamma(P,Q)}\int_{\mathcal{X}\times\mathcal{X}} d(x,y)^p \,\mathrm{d}\gamma(x,y)\Bigg)^{\!1/p},
\end{align}
Here $\Gamma(P,Q)$ denotes the set of couplings with marginals $P$ and $Q$.
\end{definition}

The Kantorovich-Rubinstein duality result gives an alternative expression for the Wasserstein-1 distance. 
\begin{definition}[Lipschitz continuity]
\label{def:lipschitz}
A function $f:(\mathcal X,\|\cdot\|)\to\mathbb R$ is said to be $L$-Lipschitz if
\[
|f(x)-f(y)|\le L\,\|x-y\|\qquad\forall\,x,y\in\mathcal X.
\]
Its global Lipschitz constant is denoted
\[
\Lip(f)\;\defeq\;\sup_{x\neq y}\frac{|f(x)-f(y)|}{\|x-y\|}.
\]
Hence $\Lip(f)\le1$ exactly when $f$ is $1$-Lipschitz.
\end{definition}

\begin{lemma}[Kantorovich-Rubinstein duality for \(W_1\) \citep{villani2009optimal}]
\label{lem:kr}
On a Polish metric space \((\mathcal{X},d)\), for probability measures \(P,Q\) with finite first moments,
\begin{align}
W_1(P,Q)
\;=\;
\sup_{\mathrm{Lip}(f)\le 1}\Big(\E_{X\sim P}[f(X)]-\E_{Y\sim Q}[f(Y)]\Big),
\end{align}
Here $\mathrm{Lip}(f)\le 1$ means $|f(x)-f(y)|\le d(x,y)$ for all $x,y\in\mathcal{X}$.

\end{lemma}

In general, TV is a stronger probability metric than Wasserstein-1 on bounded spaces, as shown by the following result. 
\begin{lemma}[Bounding \(W_1\) by TV on bounded spaces \citep{villani2009optimal}]
\label{lem:w1-tv}
If \((\mathcal{X},d)\) has finite diameter \(D \defeq \sup_{x,y} d(x,y)<\infty\), then for all probability measures  \(P\) and \(Q\),
\begin{align}
W_1(P,Q)\;\le\; D\,\mathrm{TV}(P,Q).
\end{align}
\emph{Sketch.} By Lemma~\ref{lem:kr}, scale any 1-Lipschitz \(f\) to \(\tilde f\in[0,D]\) and apply \eqref{eq:tv-sup-01}.
\end{lemma}

\subsection{BOLT and BOLT-GAN}
Throughout the appendices, we use functional forms of the learning objectives, i.e., versions in which the model parameters are suppressed. This notation is consistent with the main-paper formulation in Section~3.
\begin{definition}[BOLT functional]\label{def:bolt-functional-01}
For a measurable $h:\mathcal X\to[0,1]$, define
\begin{equation}\label{eq:bolt-functional-01}
\mathcal{L}_{\mathrm{BOLT}}(h)
\;\coloneqq\;
q_1 \;+\; q_2\,\mathbb E_{X\sim P_2}[h(X)] \;-\; q_1\,\mathbb E_{X\sim P_1}[h(X)].
\end{equation}
The universal bound implies that for any such $h$,
\begin{equation}\label{eq:bolt-bound-01}
\varepsilon_{\mathrm{Bayes}}\;\le\;\mathcal{L}_{\mathrm{BOLT}}(h).
\end{equation}
\end{definition}

\paragraph{Population risk and sample-wise BOLT loss.}
Let $C\in\{1,2\}$, $X\mid(C=i)\sim P_i$, and let $h:\mathcal X\to[0,1]$ be measurable. Define
\begin{equation}\label{eq:supp-sy-def}
s(y)\coloneqq \mathbf 1\{y=2\}-\mathbf 1\{y=1\}\in\{-1,+1\}.
\end{equation}
The sample-wise BOLT loss is
\begin{equation}\label{eq:bolt-loss-01}
\ell_{\rm BOLT}(z,y) \coloneqq s(y)\,z,
\end{equation}
and the corresponding population risk is
\begin{equation}\label{eq:bolt-risk-01}
R(h) \coloneqq \mathbb E\!\left[\ell_{\rm BOLT}(h(X),C)\right]
= -q_1\,\mathbb E_{X\sim P_1}[h(X)] + q_2\,\mathbb E_{X\sim P_2}[h(X)].
\end{equation}
Comparing with \eqref{eq:bolt-functional-01}, we can write
\begin{equation}\label{eq:bolt-functional-risk-link-01}
\mathcal{L}_{\mathrm{BOLT}}(h)=q_1+R(h),
\end{equation}
Therefore, minimizers of $R(h)$ coincide with minimizers of $\mathcal{L}_{\mathrm{BOLT}}(h)$.
\begin{definition}[Prior-weighted GAN/BOLT-GAN functional]
\label{def:bg}
Let \(g\) be a generator inducing \(P_{\mathrm{g}}\), and let \(h:\mathcal{X}\to[0,1]\) be a bounded critic.
For prior \(\pi\in(0,1)\),
\begin{align}
\mathcal{L}_{\mathrm{MB}}^{(\pi)}(g,h)
\;\defeq\;
\pi\,\E_{X\sim P_{\mathrm{data}}}[h(X)]
\;-\;
(1-\pi)\,\E_{X\sim P_g}[h(X)].
\end{align}
\end{definition}

We next record this identity, as it will be used repeatedly below.
\begin{lemma}[Complementary-sum identity]
\label{lem:complementary-sum}
For any generator \(g\) inducing \(P_{\mathrm{g}}\), any bounded critic \(h:\mathcal{X}\to[0,1]\), and any prior \(\pi\in(0,1)\), the functional \(\mathcal{L}_{\mathrm{MB}}^{(\pi)}\) satisfies
\begin{align}
\mathcal{L}_{\mathrm{MB}}^{(\pi)}(g,h)\;+\;\mathcal{L}_{\mathrm{MB}}^{(1-\pi)}(g,h)
\;=\;
\E_{X\sim P_{\mathrm{data}}}[h(X)]\; -\;\E_{X\sim P_g}[h(X)].
\end{align}
\end{lemma}
\begin{proof}
By expanding each term in Definition~\ref{def:bg} one has
\(\mathcal{L}_{\mathrm{MB}}^{(\pi)}(g,h)=\pi\,\E_{P_{\mathrm{data}}}[h]- (1-\pi)\,\E_{P_g}[h]\) and \(\mathcal{L}_{\mathrm{MB}}^{(1-\pi)}(g,h)=(1-\pi)\,\E_{P_{\mathrm{data}}}[h]- \pi\,\E_{P_g}[h]\). Adding both sides of these identities yields the claimed result.
\end{proof}

By Lemma~\ref{lem:complementary-sum}, the functionals \(\mathcal{L}_{\mathrm{MB}}^{(\pi)}\) and \(\mathcal{L}_{\mathrm{MB}}^{(1-\pi)}\) satisfy a complementary-sum identity. We use this repeatedly in Sections~3.2-3.3 of the main paper to connect the BOLT-GAN objective to the TV and Wasserstein distances.

\color{black}

\section{Proof of Theorem~\ref{bolt-map}: BOLT Recovers the MAP Classifier}
\label{sec:B}
This appendix proves Theorem~\ref{bolt-map} from the main paper: any population minimizer of the BOLT risk is Bayes-aligned, and thresholding the minimizer at $0.5$ recovers the maximum a posteriori (MAP) decision rule. We restate the theorem below without introducing a new theorem number and then give the proof. We additionally derive Bayes consistency of the resulting plug-in classifier as an auxiliary proposition.

\begin{theorem*}[Restatement of Theorem~\ref{bolt-map}]
Consider the binary classification setup in Section~\ref{sec:background}, and let $P_X$ denote the marginal distribution of $X$. Let
\begin{align}
\label{eq:supp-bolt-pop-risk}
h^\star \in \argmin_{h:\mathcal{X}\to[0,1]}
\; \mathbb{E}\bigl[\ell_{\rm BOLT}(h(X),C)\bigr],
\end{align}
and let $\eta(x):=\Pr(C=1\mid X=x)$ be the posterior probability. Then, for $P_X$-almost every $x$,
\begin{equation}
\label{eq:supp-hstar}
h^\star(x)=
\begin{cases}
1, & \eta(x)>0.5,\\
0, & \eta(x)<0.5,\\
\text{any } z\in[0,1], & \text{if }\eta(x)=0.5.
\end{cases}
\end{equation}
Consequently, the plug-in classifier
\begin{equation}
\label{eq:supp-chat}
\widehat{C}(x)=
\begin{cases}
1, & h^\star(x)\ge 0.5,\\
2, & h^\star(x)<0.5,
\end{cases}
\end{equation}
is $P_X$-almost everywhere a MAP rule; at posterior ties either label is Bayes-optimal.
\end{theorem*}

\begin{proof}
The additive constant in the BOLT bound does not affect minimizers, so minimizing the BOLT population risk is equivalent to minimizing
\begin{equation}\label{eq:supp-R-def}
R(h)\;\coloneqq\; q_2\,\mathbb{E}_{X\sim P_2}[h(X)] \;-\; q_1\,\mathbb{E}_{X\sim P_1}[h(X)].
\end{equation}
Fix $x\in\mathcal{X}$ and abbreviate $\eta=\eta(x)$. The conditional risk of predicting a value $z\in[0,1]$ is
\begin{align}
r_\eta(z)
&\coloneqq \mathbb{E}\!\left[\,\ell(z,C)\mid X=x\,\right]
 = \eta\,\ell(z,1)+(1-\eta)\,\ell(z,2). \label{eq:supp-r-eta-def}
\end{align}
Here, consistent with the main text, $\ell(z,1)=-z$ and $\ell(z,2)=+z$. Hence, we can write
\begin{equation}
\label{eq:supp-r-eta-linear}
r_\eta(z)=\eta(-z)+(1-\eta)z=(1-2\eta)z.
\end{equation}
To minimize $r_\eta(z)$ over $z\in[0,1]$:
\begin{itemize}
\item If $\eta>\tfrac{1}{2}$ then $1-2\eta<0$ and $r_\eta(z)$ is decreasing in $z$, so the minimum is at $z=1$.
\item If $\eta<\tfrac{1}{2}$ then $1-2\eta>0$ and $r_\eta(z)$ is increasing in $z$, so the minimum is at $z=0$.
\item If $\eta=\tfrac{1}{2}$ then $r_\eta(z)=0$ for all $z\in[0,1]$, so any $z\in[0,1]$ is optimal.
\end{itemize}
This yields the pointwise characterization \eqref{eq:supp-hstar}. For the plug-in rule, when $\eta(x)\neq\tfrac{1}{2}$ we have $h^\star(x)\in\{0,1\}$, and $h^\star(x)\ge\tfrac{1}{2}$ if and only if $h^\star(x)=1$, which holds if and only if $\eta(x)>\tfrac{1}{2}$. Therefore \eqref{eq:supp-chat} agrees with MAP almost everywhere, with ties only when $\eta(x)=\tfrac{1}{2}$.
\end{proof}
\begin{proposition}[Auxiliary Bayes consistency of the plug-in classifier]\label{cor:bayes-consistency-01}
Assume
\begin{equation}\label{eq:supp-no-ties}
\Pr\!\left(\eta(X)=\tfrac12\right)=0,
\end{equation}
under the distribution of $(X,C)$; here $\eta(x)=\Pr(C=1\mid X=x)$.
Let $\widehat h_n:\mathcal X\to[0,1]$ be risk-consistent for the BOLT loss
$\ell(z,1)=-z$ and $\ell(z,2)=+z$, i.e.
\begin{equation}\label{eq:supp-risk-consistency}
R(\widehat h_n)\xrightarrow[n\to\infty]{}\inf_{h:\mathcal X\to[0,1]} R(h)=R(h^\star),
\end{equation}
with
\begin{equation}\label{eq:supp-R-01}
R(h)\coloneqq \mathbb E\!\left[\ell\!\left(h(X),C\right)\right]
= -q_1\,\mathbb E_{X\sim P_1}[h(X)] + q_2\,\mathbb E_{X\sim P_2}[h(X)] .
\end{equation}
Define the plug-in classifier at threshold $1/2$,
\begin{equation}\label{eq:supp-plugin-classifier}
\widehat C_n(x)\coloneqq 2-\mathbf 1\!\left\{\widehat h_n(x)\ge \tfrac12\right\}.
\end{equation}
Then the $0$-$1$ risk converges to the Bayes risk:
\begin{equation}\label{eq:supp-01-risk-conv}
\Pr\!\left(\widehat{C}_n(X)\neq C\right)\xrightarrow[n\to\infty]{}\Pr\!\left(C_{\rm MAP}(X)\neq C\right).
\end{equation}
\end{proposition}

\begin{proof}
Let $h^\star$ be a population minimizer from Theorem~\ref{bolt-map}. Then $h^\star(x)\in\{0,1\}$ for $\eta(x)\neq \tfrac12$.
Fix $x\in\mathcal X$ and denote the posterior at this point by $\eta=\eta(x)$. The conditional risk of predicting $z\in[0,1]$ is
\begin{equation}\label{eq:supp-cond-risk}
r_\eta(z)\coloneqq \mathbb E\!\left[\ell(z,C)\mid X=x\right]
= \eta(-z)+(1-\eta)z=(1-2\eta)\,z.
\end{equation}
As in Theorem~\ref{bolt-map}, $z^\star\in\arg\min_{z\in[0,1]} r_\eta(z)$ satisfies
$z^\star=1$ if $\eta>\tfrac12$, $z^\star=0$ if $\eta<\tfrac12$, and any $z^\star\in[0,1]$ if $\eta=\tfrac12$.
Hence, for any $\widehat z\in[0,1]$,
\begin{equation}\label{eq:supp-excess-risk-pointwise}
r_\eta(\widehat z)-r_\eta(z^\star)=|1-2\eta|\cdot |\widehat z-z^\star|.
\end{equation}
Taking expectation over $X$ gives
\begin{equation}\label{eq:supp-excess-risk}
R(\widehat h_n)-R(h^\star)
=\mathbb E\!\left[\,|1-2\eta(X)|\,\big|\widehat h_n(X)-h^\star(X)\big|\,\right]\xrightarrow[n\to\infty]{}0.
\end{equation}

Let $E_n\coloneqq\{\widehat C_n(X)\neq C_{\rm MAP}(X)\}$. For any $\delta\in(0,1)$,
\begin{equation}\label{eq:supp-split-En}
E_n \subseteq \Big(E_n\cap\{|1-2\eta(X)|\ge \delta\}\Big)\ \cup\ \{|1-2\eta(X)|<\delta\}.
\end{equation}
On the event $\{|1-2\eta(X)|\ge\delta\}$, a decision mismatch (threshold $1/2$) implies
$\big|\widehat h_n(X)-h^\star(X)\big|\ge \tfrac12$ because $h^\star(X)\in\{0,1\}$ there. Therefore,
\begin{align}
R(\widehat h_n)-R(h^\star)
&\ge \mathbb E\!\left[\,|1-2\eta(X)|\cdot \big|\widehat h_n(X)-h^\star(X)\big|\cdot
\mathbf 1_{E_n\cap\{|1-2\eta(X)|\ge\delta\}}\,\right]\notag\\
&\ge \frac{\delta}{2}\,\Pr\!\left(E_n\cap\{|1-2\eta(X)|\ge\delta\}\right).
\label{eq:supp-lb-prob}
\end{align}
Sending $n\to\infty$ and using \eqref{eq:supp-excess-risk} yields
\begin{equation}\label{eq:supp-prob-goes-zero}
\Pr\!\left(E_n\cap\{|1-2\eta(X)|\ge\delta\}\right)\xrightarrow[n\to\infty]{}0.
\end{equation}
Moreover, by \eqref{eq:supp-no-ties}, $\Pr(|1-2\eta(X)|<\delta)\to 0$ as $\delta\downarrow 0$.
Combining this with \eqref{eq:supp-split-En} proves that
$\Pr(E_n)\to 0$, i.e. $\widehat C_n(X)$ agrees with $C_{\rm MAP}(X)$ with probability tending to $1$.
This implies the $0$-$1$ risk convergence \eqref{eq:supp-01-risk-conv}.
\end{proof}

\section{Proof of Corollary~\ref{cor:lr-form}: Hinge Form of the Bayes Error}\label{app:hinge-form}
\label{sec:C}
This appendix proves Corollary~\ref{cor:lr-form} from the main paper. We first restate the relevant optimizer form and tightness property for the BOLT bound, and then use it to derive the hinge representation of $\varepsilon_{\mathrm{Bayes}}$ in terms of the density ratio.

\begin{theorem}[Optimal score and tightness for Corollary~\ref{cor:lr-form}]
\label{thm:S3-opt-tight-01}
Assume $q_1,q_2>0$ and that $P_1,P_2$ admit densities $p_1,p_2$ w.r.t.\ a common dominating measure.
Define $U(x)\coloneqq p_1(x)/p_2(x)$ on $\{p_2(x)>0\}$ and $\tau\coloneqq q_2/q_1$, and let
\begin{equation}\label{eq:S3-c-01}
c(x)\defeq q_1p_1(x)-q_2p_2(x).
\end{equation}
A pointwise maximizer of $c(x)\,h(x)$ over $h(x)\in[0,1]$ is
\begin{align}
h^\star(x)\;=\;
\begin{cases}
0, & q_1p_1(x)<q_2p_2(x)\\[2pt]
1, & q_1p_1(x)\ge q_2p_2(x)
\end{cases}
\label{eq:S3-hstar-01}
\end{align}
and
\begin{equation}\label{eq:S3-tight-01}
\mathcal{L}_{\mathrm{BOLT}}(h^\star)=\varepsilon_{\mathrm{Bayes}}.
\end{equation}
Here (consistent with Def.~\ref{def:bolt-functional-01})
\begin{equation}\label{eq:S3-Lbolt-def-01}
\mathcal{L}_{\mathrm{BOLT}}(h)
\;\coloneqq\;
q_1 \;+\; q_2\,\mathbb{E}_{X\sim P_2}[h(X)] \;-\; q_1\,\mathbb{E}_{X\sim P_1}[h(X)].
\end{equation}
\end{theorem}

\begin{proof}
For fixed $x$, the map $h\mapsto c(x)h$ is affine on $[0,1]$, hence maximized at an endpoint chosen by the sign of $c(x)$:
$h^\star(x)=1$ if $c(x)\ge 0$ and $h^\star(x)=0$ if $c(x)<0$.
Since $c(x)<0$ holds if and only if $q_1p_1(x)<q_2p_2(x)$, we obtain \eqref{eq:S3-hstar-01}.

Using the definition \eqref{eq:S3-Lbolt-def-01} and substituting $h^\star$, we have
\begin{align}
\mathcal{L}_{\mathrm{BOLT}}(h^\star)
&= q_1 \;+\; q_2\,\mathbb{E}_{X\sim P_2}[h^\star(X)]
      \;-\; q_1\,\mathbb{E}_{X\sim P_1}[h^\star(X)]
\label{eq:S3-Lbolt-start-01a}\\
&= q_1 \;+\; \int \big(q_2p_2(x)-q_1p_1(x)\big)\,h^\star(x)\,\mathrm d\mu(x)
\label{eq:S3-Lbolt-int-01a}\\
&= q_1 \;+\; \int \big(q_2p_2(x)-q_1p_1(x)\big)\,\mathbf{1}\{q_1p_1(x)\ge q_2p_2(x)\}\,\mathrm d\mu(x)
\label{eq:S3-Lbolt-indic-01a}\\
&= q_1 \;-\; \int \big(q_1p_1(x)-q_2p_2(x)\big)\,\mathbf{1}\{q_1p_1(x)\ge q_2p_2(x)\}\,\mathrm d\mu(x)
\label{eq:S3-swap-01a}\\
&= q_1 \;-\; \int \big[\,q_1p_1(x)-q_2p_2(x)\,\big]_+\,\mathrm d\mu(x)
\label{eq:S3-pospart-01a}\\
&= q_1 \;-\; \int \Big(\max\{q_1p_1(x),q_2p_2(x)\}-q_2p_2(x)\Big)\,\mathrm d\mu(x)
\label{eq:S3-max-id-01a}\\
&= q_1 \;-\; \int \max\{q_1p_1(x),q_2p_2(x)\}\,\mathrm d\mu(x) \;+\; \int q_2p_2(x)\,\mathrm d\mu(x)
\label{eq:S3-finish-01a}\\
&= 1 \;-\; \int \max\{q_1p_1(x),q_2p_2(x)\}\,\mathrm d\mu(x)
\;=\; \varepsilon_{\mathrm{Bayes}}.
\label{eq:S3-bayes-01a}
\end{align}
The steps above use the following identities. First, $[a]_+=a\,\mathbf{1}\{a>0\}$, which gives the positive-part expression in \eqref{eq:S3-pospart-01a}. Second, $[b-a]_+=\max\{b,a\}-a$ with $b=q_1p_1(x)$ and $a=q_2p_2(x)$, which gives \eqref{eq:S3-max-id-01a}. Finally, $\int q_2p_2(x)\,\mathrm d\mu(x)=q_2$, which completes the proof.
\end{proof}
The cross-multiplied rule is defined also on $\{p_2=0<p_1\}$, where it chooses $h^\star=1$, and on $\{p_1=p_2=0\}$, where the value is immaterial. No absolute-continuity assumption $P_1\ll P_2$ is required. Ratio expressions below are integrated only against $P_2$; the integrand $\min\{q_1p_1,q_2p_2\}$ vanishes on $\{p_2=0\}$.

The above result gives an interesting connection between the Bayes error and the density ratio.
\begin{corollary}[Hinge representation completing Corollary~\ref{cor:lr-form}]
\label{cor:S3-hinge-01}
Define the hinge map
\begin{equation}\label{eq:S3-hinge-def}
t_0(u)\;\defeq\;[\,q_2-q_1u\,]_+ .
\end{equation}
Then, the Bayes error rate is given by
\begin{equation}\label{eq:S3-hinge-id}
\varepsilon_{\mathrm{Bayes}}
=
q_2 - \mathbb{E}_{X\sim P_2}\!\Big[t_0\big(U(X)\big)\Big].
\end{equation}
\end{corollary}

\begin{proof}
Recall
\begin{equation}\label{eq:S3-bayes-min}
\varepsilon_{\mathrm{Bayes}}
= 1-\int \max\{q_1p_1(x),q_2p_2(x)\}\,\mathrm d\mu(x)
= \int \min\{q_1p_1(x),q_2p_2(x)\}\,\mathrm d\mu(x) .
\end{equation}
Using $\min\{a,b\}=b-[\,b-a\,]_+$ with $a=q_1p_1(x)$ and $b=q_2p_2(x)$, we get
\begin{align}
\varepsilon_{\mathrm{Bayes}}
&= \int q_2p_2(x)\,\mathrm d\mu(x) - \int [\,q_2p_2(x)-q_1p_1(x)\,]_+\,\mathrm d\mu(x) \notag\\
&= q_2 - \int_{\{p_2>0\}} p_2(x)\,[\,q_2-q_1U(x)\,]_+\,\mathrm d\mu(x)
= q_2 - \mathbb{E}_{X\sim P_2}\!\Big[[\,q_2-q_1U(X)\,]_+\Big],
\end{align}
which is \eqref{eq:S3-hinge-id}.
\end{proof}
\section{Proof of Theorem~\ref{thm:ber-plug-in}: Conditional BER Estimation}
\label{app:ber-estimator}\label{sec:D}
Let $\mathcal T$ contain the data and randomness used to fit a nonnegative ratio estimate $\widehat U$. In this appendix the class priors $q_1,q_2$ are known. For GAN evaluation, the generator is frozen before fitting and evaluation; the target distributions remain fixed within each conditional experiment. Evaluation samples $X_1,\ldots,X_M$ are i.i.d.\ from $P_2$, independently of $\mathcal T$. The sample size $M$ counts the independent class-$2$ evaluation draws.

\paragraph{Bounded plug-in form.}
For $u\ge0$, define $\psi(u)=\min\{q_1u,q_2\}=q_2-[q_2-q_1u]_+$. Corollary~\ref{cor:lr-form} gives $\Bayes=\E_{P_2}\psi(U(X))$. Thus the estimator in \eqref{eq:ber-estimator-main} is exactly the hinge plug-in average, and each summand lies in $[0,q_2]$. Nonnegativity of $\widehat U$ is needed for this boundedness.

\begin{proof}[Proof of Theorem~\ref{thm:ber-plug-in}]
The clipping map $\psi$ is $q_1$-Lipschitz. Conditional on $\mathcal T$, the ratio is fixed, so independence gives
\begin{align}
\left|\E[\widehat\varepsilon_{\rm BOLT}\mid\mathcal T]-\Bayes\right|
&=\left|\E_{P_2}\big[\psi(\widehat U(X))-\psi(U(X))\big]\right|\\
&\le q_1\E_{P_2}|\widehat U(X)-U(X)|\le q_1\varepsilon_0.
\end{align}
A random variable in $[0,q_2]$ has variance at most $q_2^2/4$. The $M$ summands are conditionally independent, hence
\begin{equation}
\Var(\widehat\varepsilon_{\rm BOLT}\mid\mathcal T)
=\frac{1}{M}\Var_{P_2}\big(\psi(\widehat U(X))\big)
\le\frac{q_2^2}{4M}.
\end{equation}
This proves both claims. The same boundedness also yields the conditional Hoeffding bound
$\Pr(|\widehat\varepsilon_{\rm BOLT}-\E[\widehat\varepsilon_{\rm BOLT}\mid\mathcal T]|\ge t\mid\mathcal T)\le2\exp(-2Mt^2/q_2^2)$.
\end{proof}

\paragraph{Combined conditional error bound.}
For $0<\delta<1$, the bias bound and Hoeffding inequality imply, with conditional probability at least $1-\delta$,
\begin{equation}
|\widehat\varepsilon_{\rm BOLT}-\Bayes|
\le q_1\varepsilon_0+q_2\sqrt{\frac{\log(2/\delta)}{2M}}.
\end{equation}
A computable confidence interval follows when a valid bound on $\varepsilon_0$ is supplied independently. The variance constant provides a worst-case guarantee over bounded variables.

\paragraph{Variation across fitted ratios.}
By the law of total variance,
\begin{equation}
\Var(\widehat\varepsilon_{\rm BOLT})=
\E\!\left[\Var(\widehat\varepsilon_{\rm BOLT}\mid\mathcal T)\right]
+\Var\!\left(\E[\widehat\varepsilon_{\rm BOLT}\mid\mathcal T]\right).
\end{equation}
The first term is bounded by $q_2^2/(4M)$ and describes evaluation noise. The second term captures variability in the fitted ratio; its rate, together with the approximation error $\varepsilon_0$, depends on the ratio estimator and its training assumptions. The theorem treats the class priors as known.

\paragraph{Obtaining the ratio.}
The endpoint-valued BOLT optimizer encodes the MAP decision. A separate calibrated posterior model $\widehat\eta(x)$ trained under priors $q_1,q_2$ may supply
\begin{equation}
\widehat U(x)=\frac{q_2}{q_1}\frac{\widehat\eta(x)}{1-\widehat\eta(x)}.
\end{equation}
Posterior clipping away from $1$ ensures finite values and contributes to the approximation error $\varepsilon_0$. The evaluation split must be independent of fitting and model selection. In a synthetic study the known class densities can instead supply $U$ exactly.

\section{Further Results on BOLT-GAN}
\label{sec:E}

This appendix gives further properties of the BOLT-GAN objective $\mathcal{L}_{\mathrm{MB}}^{(\pi)}(g,h)$ defined in Section~\ref{sec:max-ber-game} of the main paper. We establish elementary bounds, monotonicity in the prior $\pi$, and the relation to total variation and to the Wasserstein-1 distance under a $1$-Lipschitz constraint. The latter completes the proof of Theorem~\ref{thm:lipschitz-bolt-w1} from the main paper.

\paragraph{A consequence for general priors.}
The following corollary follows directly from Theorem~\ref{thm:bolt-vs-tv}.
\begin{corollary}[Half-TV guarantee]
\label{cor:halftv}
Let $\mathcal{D}^{(\pi)}$ be defined as in Theorem~\ref{thm:bolt-vs-tv}. If for a generator $g$,
\begin{equation}
\max\{\mathcal{D}^{(\pi)}(g),\mathcal{D}^{(1-\pi)}(g)\}\leq \varepsilon,
\end{equation}
then $\mathrm{TV}(P_{\mathrm{data}},P_{\mathrm g})\leq 2\varepsilon$.
\end{corollary}
\begin{proof}
By Theorem~\ref{thm:bolt-vs-tv},
\begin{equation}
\mathrm{TV}(P_{\mathrm{data}},P_{\mathrm g})
\le \mathcal{D}^{(\pi)}(g)+\mathcal{D}^{(1-\pi)}(g)
\le 2\max\{\mathcal{D}^{(\pi)}(g),\mathcal{D}^{(1-\pi)}(g)\}.
\end{equation}
The assumption then gives the claim.
\end{proof}

\paragraph{Basic bounds.}
The bounded range $h\in[0,1]$ gives a finite scale for the prior-weighted functional, which can nevertheless be negative or zero before maximization. We first record the corresponding elementary bounds.

\begin{lemma}[Lower and upper bounds]
\label{lem:lower-upper}
For any generator \(g\) and any bounded critic \(h:\mathcal X\to[0,1]\),
\begin{align}
\pi-1 \;\le\; \mathcal{L}_{\mathrm{MB}}^{(\pi)}(g,h) \;\le\; \pi.
\end{align}
Both bounds are tight in the sense that there exists a pair \((g,h)\) attaining them.
\end{lemma}

\begin{proof}
Since \(0\le \E_{P_{\mathrm{data}}}[h]\le 1\) and \(0\le \E_{P_g}[h]\le 1\),
\(\mathcal{L}_{\mathrm{MB}}^{(\pi)}(g,h)=\pi\,\E_{P_{\mathrm{data}}}[h]-(1-\pi)\,\E_{P_g}[h]\in[\pi-1,\;\pi]\).
Tightness follows by choosing distributions and a measurable \(h\) that separate supports:
if \(h=1\) on \(\operatorname{supp}(P_{\mathrm{data}})\) and \(h=0\) on \(\operatorname{supp}(P_{\mathrm{g}})\), the value is \(\pi\);
if \(h=0\) on \(\operatorname{supp}(P_{\mathrm{data}})\) and \(h=1\) on \(\operatorname{supp}(P_{\mathrm{g}})\), it is \(\pi-1\).
\end{proof}

\paragraph{Monotonicity in the prior.}
We next show that the maximized gap grows with \(\pi\).

\begin{lemma}[Monotonicity in \(\pi\)]
\label{lem:supp-monotone}
Fix a generator \(g\) with density \(q\) and write \(p\) for the data density w.r.t.\ a common dominating measure.
Let
\begin{align}
\mathcal{D}^{(\pi)}(g)\;\defeq\;\max_{h:\,\mathcal X\to[0,1]}\;\mathcal{L}_{\mathrm{MB}}^{(\pi)}(g,h)
\;=\;\int_{\mathcal X}\!\big[\pi\,p(x)-(1-\pi)\,q(x)\big]_+\,\mathrm d\mu(x).
\end{align}
Then \(\mathcal{D}^{(\pi)}(g)\) is nondecreasing in \(\pi\in[0,1]\).
\end{lemma}

\begin{proof}
For each fixed \(x\), the map \(\pi\mapsto [\,\pi p(x)-(1-\pi)q(x)\,]_+\) is nondecreasing:
its derivative is \(p(x)+q(x)>0\) wherever the bracket is positive and \(0\) elsewhere.
We obtain the desired result by integrating over \(\mathcal X\).
\end{proof}

\paragraph{Balanced vs.\ prior-weighted gap.}
We now relate the prior-weighted objective to its balanced counterpart on a $1$-Lipschitz class. Let us recap the following definitions:
\begin{align}
\mathcal{H}_{\mathrm{Lip}}
&\;\defeq\;\big\{\,h:\mathcal{X}\!\to\![0,1] \;:\;
|h(x)-h(y)| \le \|x-y\| \ \forall\,x,y\in\mathcal{X}\,\big\},
\label{eq:HLip}\\[4pt]
\mathcal{D}^{(\pi)}_{\mathrm{Lip}}(g)
&\;\defeq\;
\sup_{h\in\mathcal{H}_{\mathrm{Lip}}}\Big\{
\pi\E_{X\sim P_{\mathrm{data}}}\!\big[h(X)\big]
- (1-\pi)\E_{X\sim P_g}\!\big[h(X)\big]\Big\},
\label{eq:Dpi-def}\\[4pt]
\Sigma_{\mathrm{Lip}}(g)
&\;\defeq\;
\sup_{h\in\mathcal{H}_{\mathrm{Lip}}}
\Big(\E_{X\sim P_{\mathrm{data}}}\!\big[h(X)\big]
-\E_{X\sim P_g}\!\big[h(X)\big]\Big).
\label{eq:Sigma-def}
\end{align}

The next inequality is the pointwise scalar comparison we use repeatedly.

\begin{lemma}[Pointwise dominance]
\label{lem:pointwise-dominance}
For $0<\pi\le1/2$ and $a,b\in[0,1]$,
\begin{equation}
\pi a-(1-\pi)b=\pi(a-b)-(1-2\pi)b\le\pi(a-b).
\end{equation}
\end{lemma}
\begin{proof}
The subtracted term is nonnegative because $b\ge0$ and $1-2\pi\ge0$.
\end{proof}

\begin{lemma}[From pointwise to functional dominance]
\label{lem:functional-dominance}
For $0<\pi\le1/2$,
\begin{equation}
0\le\mathcal D_{\rm Lip}^{(\pi)}(g)\le\pi\Sigma_{\rm Lip}(g).
\end{equation}
\end{lemma}
\begin{proof}
Apply Lemma~\ref{lem:pointwise-dominance} to $a=\E_{P_{\rm data}}[h]$, $b=\E_{P_{\rm g}}[h]$ and take suprema over $h\in\mathcal H_{\rm Lip}$. The admissible score $h\equiv0$ gives the lower bound.
\end{proof}

\paragraph{Connection to Wasserstein-$1$ distance.}
Kantorovich--Rubinstein duality and function-class inclusion give $\Sigma_{\rm Lip}(g)\le W_1(P_{\rm data},P_{\rm g})$. The preceding lemma thus gives the same $\pi$ factor as Theorem~\ref{thm:lipschitz-bolt-w1}; no complementary-prior relaxation is needed.

\paragraph{Case $\pi>1/2$ and imbalanced priors.}
Write $D_\pi(P,Q)=\sup_{h\in\mathcal H_{\rm Lip}}\{\pi\E_Ph-(1-\pi)\E_Qh\}$. The change $h\mapsto1-h$ gives the exact identity
\begin{equation}
D_\pi(P,Q)=(2\pi-1)+D_{1-\pi}(Q,P).
\end{equation}
Consequently, $0\le D_\pi(P,Q)-(2\pi-1)\le(1-\pi)W_1(P,Q)$ for $\pi>1/2$, with the offset retained under relabeling. In particular, $D_\pi(P,P)=\max\{2\pi-1,0\}$. A single imbalanced gap can vanish for distinct distributions: on two points, take $P=(1/2,1/2)$, $Q=(2/5,3/5)$, and $\pi=2/5$. Then $\pi P\le(1-\pi)Q$ coordinatewise, so even the unrestricted gap is zero, although $\mathrm{TV}(P,Q)=1/10$.

\subsection{Proof of the total-variation relation}
\label{app:tv-proof}
\begin{proof}
By Lemma~\ref{lem:complementary-sum},
\begin{equation}
\mathcal{L}_{\mathrm{MB}}^{(\pi)}(g,h)+\mathcal{L}_{\mathrm{MB}}^{(1-\pi)}(g,h)
=\mathbb{E}_{P_{\mathrm{data}}}[h(X)]-\mathbb{E}_{P_{\mathrm g}}[h(X)].
\end{equation}
Hence,
\begin{align}
\Sigma(g)
&\triangleq \sup_{h:\mathcal X\to[0,1]}
\big[\mathcal{L}_{\mathrm{MB}}^{(\pi)}(g,h)+\mathcal{L}_{\mathrm{MB}}^{(1-\pi)}(g,h)\big] \\
&= \sup_{h:\mathcal X\to[0,1]}\int h(X)\,\mathrm d(P_{\mathrm{data}}-P_{\mathrm g})
= \mathrm{TV}(P_{\mathrm{data}},P_{\mathrm g}),
\end{align}
where the last identity follows from Lemma~\ref{lem:tv-forms}. Using the inequality $\sup_h\{F(h)+G(h)\}\le \sup_hF(h)+\sup_hG(h)$, we can write
\begin{align}
\mathrm{TV}(P_{\mathrm{data}},P_{\mathrm g})=\Sigma(g)
&\le \sup_h\mathcal{L}_{\mathrm{MB}}^{(\pi)}(g,h)+\sup_h\mathcal{L}_{\mathrm{MB}}^{(1-\pi)}(g,h)
=\mathcal{D}^{(\pi)}(g)+\mathcal{D}^{(1-\pi)}(g).
\end{align}
For $\pi=0.5$, the two terms inside the supremum are identical, and therefore
\begin{equation}
\mathrm{TV}(P_{\mathrm{data}},P_{\mathrm g})=\sup_h 2\mathcal{L}_{\mathrm{MB}}^{(0.5)}(g,h)=2\mathcal{D}^{(0.5)}(g).
\end{equation}
This concludes the proof.
\end{proof}

\subsection{Proof of the Wasserstein upper bound}
\label{app:w1-proof}
\begin{proof}
For every $h\in\mathcal H_{\rm Lip}$ and $0<\pi\le1/2$,
\begin{align}
\mathcal L_{\rm MB}^{(\pi)}(g,h)
&=\pi\big(\E_{P_{\rm data}}[h]-\E_{P_{\rm g}}[h]\big)
 -(1-2\pi)\E_{P_{\rm g}}[h]\\
&\le\pi\big(\E_{P_{\rm data}}[h]-\E_{P_{\rm g}}[h]\big).
\end{align}
Since a bounded $1$-Lipschitz function is admissible in the unrestricted Wasserstein dual, taking the supremum gives
\begin{equation}
\mathcal D_{\rm Lip}^{(\pi)}(g)\le\pi\Sigma_{\rm Lip}(g)
\le\pi W_1(P_{\rm data},P_{\rm g}).
\end{equation}
Finite first moments justify the dual expectations. Finally, $h\equiv0$ is admissible, so $\mathcal D_{\rm Lip}^{(\pi)}(g)\ge0$. This proves both inequalities of Theorem~\ref{thm:lipschitz-bolt-w1} with exactly the stated constant.
\end{proof}

\subsection{Balanced bounded-domain comparison}
\label{app:bounded-domain}
Assume $\mathcal X\subseteq\mathbb R^d$ is a nonempty bounded domain with diameter $d_{\mathcal X}$, and let $P,Q$ be Borel probability measures on it. Set $c=\max\{1,d_{\mathcal X}\}$ and
\begin{equation}
D_B(P,Q)=\frac12\sup_{h\in\mathcal H_{\rm Lip}}(\E_Ph-\E_Qh).
\end{equation}
The class is closed under $h\mapsto1-h$, so the supremum is unchanged by placing an absolute value around the expectation difference. The lower inequality $2D_B\le W_1$ follows by function-class inclusion. For the upper inequality, any real-valued $1$-Lipschitz potential $f$ has oscillation at most $d_{\mathcal X}$. Its infimum $a=\inf_{x\in\mathcal X}f(x)$ is finite, and $h=(f-a)/c$ lies in $[0,1]$ with Lipschitz constant at most $1$. Since $P$ and $Q$ have equal total mass,
\begin{equation}
\E_Pf-\E_Qf=c(\E_Ph-\E_Qh)\le2cD_B(P,Q).
\end{equation}
Taking the supremum over $f$ in the Kantorovich--Rubinstein dual proves \eqref{eq:bounded-domain-converse}. Thus $D_B(P,Q_n)\to0$ if and only if $W_1(P,Q_n)\to0$ on this bounded domain. This equivalence compares the population discrepancies optimized over their full function classes.

\paragraph{Scale and dimensionality.}
Here $d_{\mathcal X}=\sup_{x,y\in\mathcal X}\|x-y\|_2$. For the full image cube $[-1,1]^d$, it is $2\sqrt d$ (with channels included in $d$). The converse factor $2\max\{1,d_{\mathcal X}\}$ therefore scales with the image dimension. Rescaling the input metric changes both the Wasserstein distance and the admissible Lipschitz class.

\subsection{Binary excess-risk transfer}
\label{app:calibration}
Let $R_B(h)=\E[\ell_{\rm BOLT}(h(X),C)]$, $R_B^*=\inf_{h:\mathcal X\to[0,1]}R_B(h)$, and $C_h=1$ when $h\ge1/2$, with $C_h=2$ otherwise. For $\eta(x)\ne1/2$, the conditional BOLT risk is $(1-2\eta(x))h(x)$ and its endpoint optimizer is $h^*(x)=\mathbf1\{\eta(x)>1/2\}$. Its excess is
\begin{equation}
|2\eta(x)-1|\,|h(x)-h^*(x)|.
\end{equation}
Whenever $C_h$ differs from the unique MAP decision, $|h-h^*|\ge1/2$. The conditional excess classification error is $|2\eta-1|$ on that event and zero otherwise. At posterior ties both excess risks vanish. Integrating gives
\begin{equation}
R_{0\text{-}1}(C_h)-\Bayes\le2\{R_B(h)-R_B^*\}.
\end{equation}
The unrestricted optimum is the comparator. A finite constrained class may have a positive approximation gap, and numerical optimization may add another gap.

\paragraph{An explicit smoothness--tightness trade-off.}
Let $P_1=\operatorname{Unif}(0,1)$ and $P_2=\operatorname{Unif}(-1,0)$, with balanced priors. The unrestricted BER is zero. For any $1$-Lipschitz $h:[-1,1]\to[0,1]$,
\begin{equation}
\E_{P_1}[h]-\E_{P_2}[h]
=\int_0^1\{h(x)-h(-x)\}\,\mathrm dx
\le\int_0^1\min\{2x,1\}\,\mathrm dx=\tfrac34.
\end{equation}
Thus the best bounded-Lipschitz BER upper bound is at least $1/2-(1/2)(3/4)=1/8$. The ramp $h(x)=\min\{1,\max\{0,1/2+x\}\}$ attains it. Thresholding this ramp gives zero classification error despite the positive bound-approximation error. Smoothness replaces the abrupt endpoint switch by a transition region with nonzero input gradients and flat regions outside it. The excess-risk inequality above establishes Bayes-risk consistency even in the presence of posterior ties; an almost-surely unique MAP label additionally gives convergence to that label selection.

\subsection{Sigmoid gradients and regularizer scope}
\label{app:gradient-tradeoff}
\paragraph{The ideal sigmoid class.}
For $L>0$, define $\mathcal H_{\sigma,L}=\{\sigma\circ f:f:\mathcal X\to\mathbb R,\ \operatorname{Lip}(f)\le L\}$ and let $D_{\sigma,L}^{(\pi)}$ be the supremum of $\mathcal L_{\rm MB}^{(\pi)}$ over that class. Each score has Lipschitz constant at most $L/4$, so the same expansion as in Appendix~\ref{app:w1-proof} gives, for $\pi\le1/2$,
\begin{equation}
0\le D_{\sigma,L}^{(\pi)}(P,Q)\le\frac{\pi L}{4}W_1(P,Q).
\end{equation}
The lower endpoint is approached by sending a constant raw score to $-\infty$. At $L=1$ this yields the factor $\pi/4$. The argument applies to the full ideal sigmoid class with a global raw-function Lipschitz bound; the converse in Appendix~\ref{app:bounded-domain} concerns the full bounded class $\mathcal H_{\rm Lip}$.

At fixed critic parameters, let $x=g_\varphi(Z)$ and $J_\varphi g_\varphi(Z)$ denote the generator Jacobian. Wherever the chain rule and differentiation under the expectation apply (for example, under an integrable dominating gradient),
\begin{equation}
\nabla_\varphi\mathcal L_\varphi
=-(1-\pi)\E_Z\!\left[
\sigma'(f_\theta(x))\,J_\varphi g_\varphi(Z)^\top\nabla_x f_\theta(x)
\right].
\end{equation}
Compared with the raw-score WGAN-GP term, the extra factor satisfies $0<\sigma'(f)\le1/4$ for finite $f$ and tends to zero at either extreme. It attenuates large-logit contributions and weakens generator gradients in saturated regions. A different bounded activation changes the slope and effective gradient scale even when its range is unchanged.

Interpolation GP penalizes deviations of the raw-score input-gradient norm from $1$ on real--fake interpolants. R1 and R2 instead use the zero-centered penalties $\frac\gamma2\E_{P_{\rm data}}\|\nabla_x f_\theta\|_2^2$ and $\frac\gamma2\E_{P_{\rm g}}\|\nabla_x f_\theta\|_2^2$~\citep{mescheder2018which}. These penalties provide alternative forms of local regularization for the BOLT objective. Comparisons among them depend on the update ratio, coefficient, architecture, and training budget, whereas the Wasserstein bounds above assume a globally Lipschitz discriminator class.

\paragraph{Fixed-critic sampling bounds.}
For fixed $g,h$ with $0\le h\le1$, independently sample $B$ real and $B$ generated observations. The mini-batch functional satisfies
\begin{equation}
\operatorname{Var}(\widehat{\mathcal L}_{\rm MB}^{(\pi)}\mid g,h)
\le\frac{\pi^2+(1-\pi)^2}{4B}.
\end{equation}
Indeed, each score has variance at most $1/4$, and independence makes the weighted variances add. Bounded independent summands also give
\begin{equation}
\Pr\!\left(\left|\widehat{\mathcal L}_{\rm MB}^{(\pi)}-\mathcal L_{\rm MB}^{(\pi)}\right|\ge t\mid g,h\right)
\le2\exp\!\left(-\frac{2Bt^2}{\pi^2+(1-\pi)^2}\right).
\end{equation}
These sampling bounds quantify the Monte Carlo error of the unpenalized adversarial functional, conditional on generator and critic functions fixed independently of the evaluation batch.

\section{Implementation of BOLT-GAN}
\label{sec:F}
This appendix records the theoretical preliminaries and implementation details for Lipschitz enforcement in BOLT-GAN. We first summarize basic Lipschitz facts used by the critic, then describe three enforcement mechanisms: spectral normalization, gradient penalty, and weight clipping. We finally list related penalties and diagnostics used to monitor training. The full $1$-Lipschitz class is used in the Wasserstein dual (Lemma~\ref{lem:kr}); the neural implementation uses sampled gradient regularization as a practical approximation. In our implementation, the gradient penalty is applied to the \emph{raw} critic output $f_\theta$ (pre-activation), and the bounded score $h_\theta=\sigma(f_\theta)\in[0,1]$ enters the BOLT objective, as summarized in Algorithm~\ref{alg:bolt-gan}~\citep{gulrajani2017improved,miyato2018spectral}.

\subsection{Preliminaries on Lipschitz continuity}
\label{subsec:lipschitz-prelims}
\begin{lemma}[Basic Lipschitz calculus \citep{miyato2018spectral}]
\label{lem:S5-lip-calculus}
Let $f$, $g$, $f_1$, and $f_2$ be Lipschitz. Then
\begin{align}
\mathrm{Lip}(g\circ f)&\le \mathrm{Lip}(g)\,\mathrm{Lip}(f), 
\label{eq:S5-comp}\\
\mathrm{Lip}(f_1+f_2)&\le \mathrm{Lip}(f_1)+\mathrm{Lip}(f_2).
\label{eq:S5-sum}
\end{align} 
\end{lemma}

\begin{lemma}[Layerwise Lipschitz constants in common architectures \citep{miyato2018spectral,anil2019sorting}]
\label{lem:S5-layerwise}
Assume all vector spaces are equipped with the Euclidean norm $\|\cdot\|_2$, and for a function $f$ we write $\mathrm{Lip}(f)$ for its global Lipschitz constant with respect to $\|\cdot\|_2$.

\begin{enumerate}[label=(\roman*), leftmargin=*, itemsep=2pt]
\item Linear / convolutional layers.
For an affine map $T(x)=Wx+b$ (fully connected) or a convolutional layer viewed as a linear map $x\mapsto Wx$,
\[
\mathrm{Lip}(T)=\|W\|_{2},
\]
i.e., the operator (spectral) norm of $W$. Biases do not affect $\mathrm{Lip}$.

\item Pointwise activations (elementwise maps).
For an elementwise locally absolutely continuous scalar activation with essentially bounded derivative, $\mathrm{Lip}(\phi)=\operatorname*{ess\,sup}_{t\in\mathbb R}|\phi'(t)|$.

\item Residual blocks.
For a residual block of the form $x\mapsto Sx+F(x)$ with a linear skip $S$ (identity or $1{\times}1$ projection),
\[
\mathrm{Lip}(x\mapsto Sx+F(x)) \;\le\; \|S\|_{2}+\mathrm{Lip}(F).
\]
In particular, if $S=I$, then $\mathrm{Lip}\le 1+\mathrm{Lip}(F)$.

\item Resampling operators.
If $S$ is a (linear) down/up-sampling operator used before or after a block $F$, then
\[
\mathrm{Lip}(F\circ S)\le \mathrm{Lip}(F)\,\|S\|_2,\qquad
\mathrm{Lip}(S\circ F)\le \|S\|_2\,\mathrm{Lip}(F).
\]
Common strided convolutions, pooling, and interpolation are linear maps with finite $\|S\|_2$ that must be accounted for in per-layer bounds.
\end{enumerate}
\end{lemma}

\begin{remark}[Standard elementwise activations and their Lipschitz constants]
\label{rem:activations-lip}
Several common activations are $1$-Lipschitz. ReLU and tanh both have Lipschitz constant $1$. LeakyReLU with negative slope $\alpha$ has Lipschitz constant $\max\{1,|\alpha|\}$. The sigmoid activation has derivative $\sigma(t)(1-\sigma(t))$, which is at most $1/4$; hence sigmoid is $1/4$-Lipschitz. These constants are used when composing layerwise Lipschitz bounds.
\end{remark}

\subsection{Spectral Normalization (SN)}
\label{subsec: sn}
\emph{Method.} SN rescales a linear map by an estimate of its top singular value, usually using power iteration~\citep{miyato2018spectral}. For a convolution, the relevant Euclidean operator norm includes the spatial action of the complete convolution map, whereas flattened-kernel normalization acts on a reshaped weight matrix. Finite power iteration provides an estimated normalization factor.

\emph{Layerwise composition.} For a serial composition of maps whose true operator norms and activation Lipschitz constants are at most $1$, the global constant is at most $1$. This construction accounts for every resampling, normalization, pooling, and residual operation.

\emph{Residual blocks and skips.} The Lipschitz bound of a residual sum includes both branches. If each branch is $1$-Lipschitz, $Sx+F(x)$ can be $2$-Lipschitz, and scaling only $F$ to constant $c$ gives a depth-$L$ bound $(1+c)^L$. A sufficient $1$-Lipschitz construction is $(Sx+F(x))/2$ when both branches are $1$-Lipschitz, or division by a valid upper bound on $\|S\|_2+\operatorname{Lip}(F)$.

\emph{Pooling.} For $s$ spatial sites with Euclidean/Frobenius input norm, channelwise sum pooling has operator norm $\sqrt s$, while average pooling has norm $1/\sqrt s$. This follows from Cauchy--Schwarz, with equality for equal features at all sites. The global product bound therefore includes the sum-pooling factor together with the linear-head norm.

\subsection{Gradient Penalty (GP)}
\label{subsec: gp}
\emph{Method.} Gradient penalty encourages \(\|\nabla_x f_\theta(x)\|_2\approx 1\) on a chosen sampling distribution. The most common choice is the \emph{interpolation penalty} \citep{gulrajani2017improved}: draw \((x_{\mathrm{real}},x_{\mathrm{fake}})\), form \(\tilde x=\alpha x_{\mathrm{real}}+(1-\alpha)x_{\mathrm{fake}}\) with \(\alpha\sim\mathrm{Unif}[0,1]\), and add
\begin{align}
\lambda\,\E_{\tilde x}\!\big(\,\|\nabla_{\tilde x}f_\theta(\tilde x)\|_2-1\,\big)^2
\label{eq:S5-gp}
\end{align}
to the critic loss. Applying the penalty to the \emph{pre-activation} \(f_\theta\) preserves the raw input-gradient scale before the \(1/4\)-Lipschitz sigmoid.

\emph{Local regularization.} The penalty \eqref{eq:S5-gp} is a \emph{soft} constraint. A zero population penalty gives gradient norm $1$ almost surely under the interpolation distribution, while its empirical counterpart controls the sampled interpolants~\citep{gulrajani2017improved}.

\emph{Implementation.} Differentiating the gradient penalty requires retaining the computation graph through \(\|\nabla_{\tilde x}f_\theta(\tilde x)\|_2\) for higher-order derivatives. The default critic omits BatchNorm; layer or instance normalization, when used, contributes its own factor to a global Lipschitz bound. Algorithm~\ref{alg:bolt-gan} applies GP at every critic update. Appendix~\ref{sec:G1} evaluates \(\lambda\in\{1,5,10,20\}\), and Appendix~\ref{subsec: diagnostics} defines the penalty-to-adversarial-term ratio used to characterize the balance between the two terms.

\subsection{Weight clipping}
\label{subsec: clipping}
\emph{Method.} Weight clipping applies \(w\leftarrow\mathrm{clip}(w,\,-L_{\mathrm{wc}},\,L_{\mathrm{wc}})\) after each parameter update~\citep{arjovsky2017wasserstein}. The resulting element-wise bounds give coarse bounds on the per-layer operator norms.

\emph{Guarantee (coarse).} If a layer \(W\in\mathbb R^{m\times n}\) is element-wise clipped as \(|W_{ij}|\le L_{\mathrm{wc}}\), then
\begin{align}
\|W\|_2 \;\le\; \|W\|_{\mathrm{F}}\;\le\; \sqrt{mn}\,L_{\mathrm{wc}},
\label{eq:S5-clip-bound}
\end{align}
so a depth-\(L\) network with \(1\)-Lipschitz activations satisfies 
\(\mathrm{Lip}(f)\le \prod_{\ell=1}^L \sqrt{m_\ell n_\ell}\,L_{\mathrm{wc}}\).
This yields at best a \emph{coarse} global bound and often harms capacity. The bound must additionally include residual and resampling factors~\citep{arjovsky2017wasserstein,miyato2018spectral}.

\subsection{Other gradient penalties and related methods}
\label{subsec: other}
\begin{itemize}[leftmargin=*]
\item One-sided GP (Lipschitz penalty). Penalize only \(\|\nabla f_\theta\|_2>1\): \(\lambda\,\E[\max(0,\|\nabla f_\theta\|-1)^2]\), without pushing gradients smaller than $1$ up toward $1$ \citep{petzka2018lipsgan}.
\item Zero-centered penalties \(R_1/R_2\). Apply \(\lambda\,\E_{x\sim P_{\mathrm{data}}}\|\nabla_x f_\theta(x)\|_2^2\) (\(R_1\)) or \(\lambda\,\E_{x\sim P_g}\|\nabla_x f_\theta(x)\|_2^2\) (\(R_2\)) to regularize critic gradients around real and generated samples \citep{mescheder2018which}.
\item DRAGAN (local penalties). Sample \(\tilde x=x_{\mathrm{real}}+\delta\) with small noise and penalize \(\|\nabla_{\tilde x}f_\theta(\tilde x)\|_2\) near the data manifold to promote local smoothness \citep{kodali2017dragan}.
\item Orthogonal/Parseval constraints. Constrain \(W^\top W\approx I\) (Parseval/Björck) to shrink \(\|W\|_2\) toward \(1\) at higher compute cost \citep{cisse2017parseval}.
\item Lipschitz activations. Using activations with known Lipschitz constants (e.g., GroupSort) can yield provably \(1\)-Lipschitz networks when combined with spectral control of linear layers \citep{anil2019sorting}.
\item Optimizer-level gradient clipping limits parameter-gradient magnitudes and complements function-level regularization through SN or GP.
\end{itemize}
These alternatives differ in how strictly they enforce the Lipschitz constraint versus how much computational cost or flexibility they introduce.

\subsection{Diagnostics and failure modes}
\label{subsec: diagnostics}
Input-gradient norms on sampled interpolants characterize the local action of the penalty. Their distribution, together with the raw adversarial term and weighted GP, describes the balance between discrimination and regularization.

For a sampled unweighted penalty $\widehat A=B^{-1}\sum_j(\|\nabla_{\tilde x_j}f_\theta(\tilde x_j)\|_2-1)^2$ and a fixed $\epsilon>0$, a finite diagnostic is
\begin{equation}
r_{\rm GP}=\frac{\lambda_{\rm GP}\widehat A}{|\widehat{\mathcal L}_{\rm MB}^{(\pi)}|+\epsilon},
\qquad \log(1+r_{\rm GP}).
\label{eq:stable-gp-diagnostic}
\end{equation}
The absolute value and positive $\epsilon$ make this diagnostic well-defined for finite losses. Its interpretation depends on the chosen $\epsilon$ and the separate numerator and denominator: a large ratio can arise when the adversarial gap is near zero. A global function-level bound additionally accounts for the complete residual, pooling, and normalization operator chain.

\normalfont
\section{Additional Experimental Results}
\label{sec:G}

\subsection{Representative Qualitative Samples}
\label{app:qualitative-samples}

Figure~\ref{fig:main_qualitative} shows representative \BOLTGAN samples from all four image-generation datasets.

\begin{figure}[!htbp]
  \centering
  \begin{subfigure}[t]{0.48\linewidth}
    \centering
    \includegraphics[width=\linewidth]{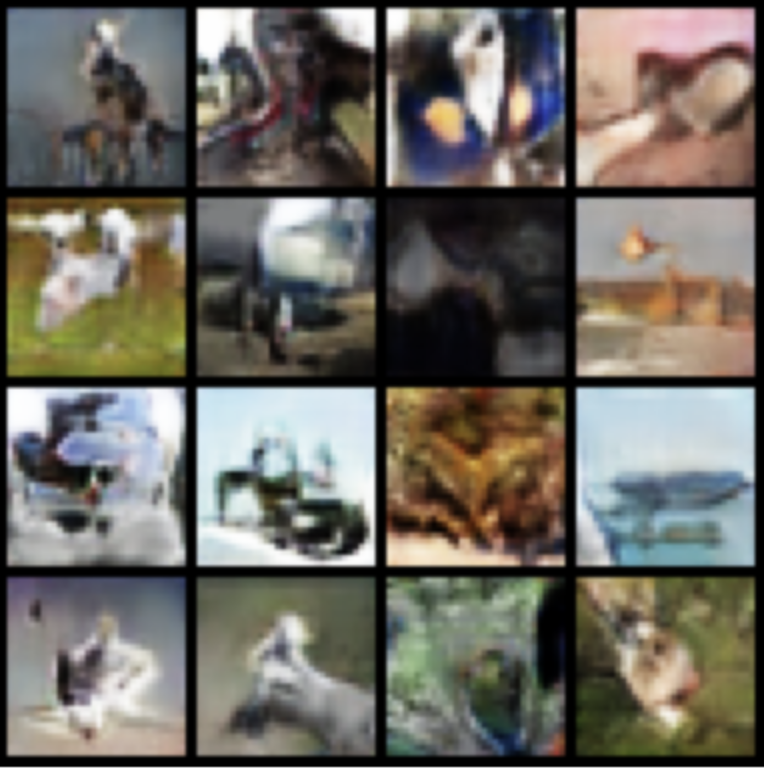}
    \caption{CIFAR-10}
  \end{subfigure}\hfill
  \begin{subfigure}[t]{0.48\linewidth}
    \centering
    \includegraphics[width=\linewidth]{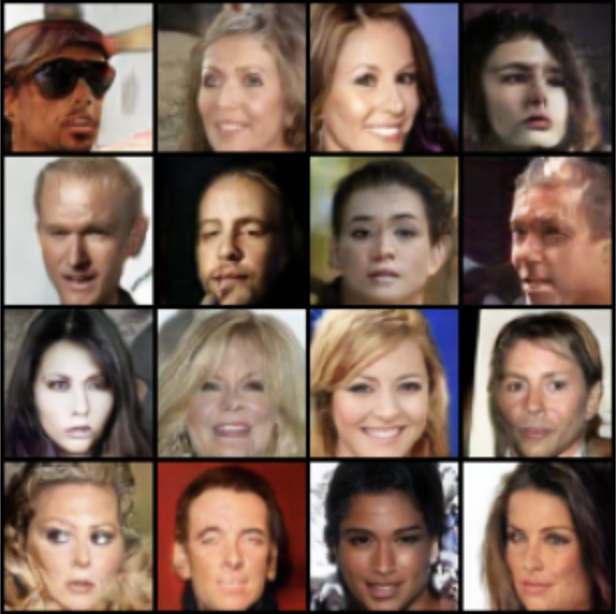}
    \caption{CelebA-64}
  \end{subfigure}

  \vspace{0.6em}
  \begin{subfigure}[t]{0.48\linewidth}
    \centering
    \includegraphics[width=\linewidth]{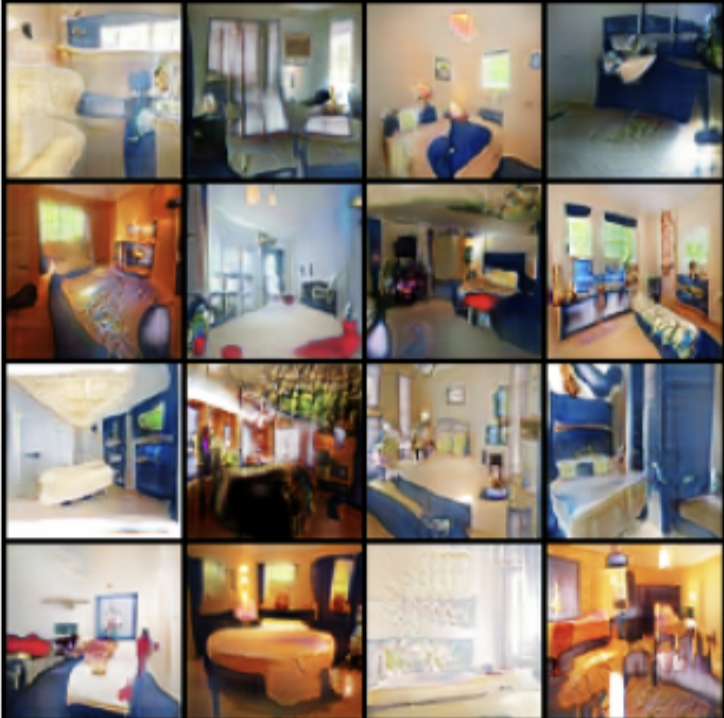}
    \caption{LSUN Bedroom-64}
  \end{subfigure}\hfill
  \begin{subfigure}[t]{0.48\linewidth}
    \centering
    \includegraphics[width=\linewidth]{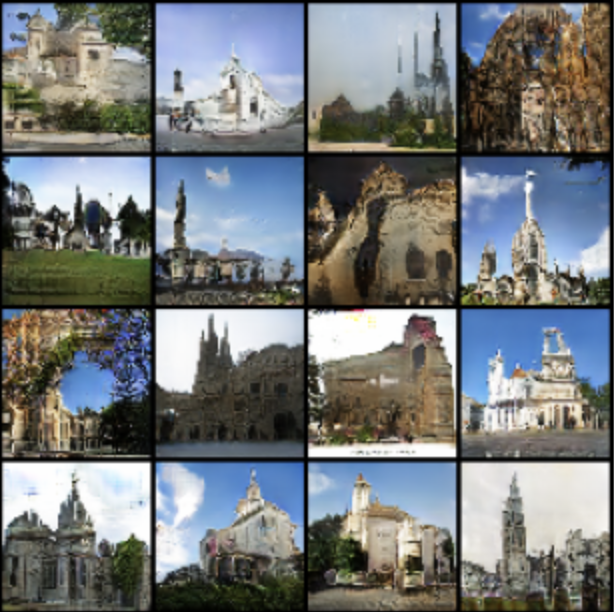}
    \caption{LSUN Church-64}
  \end{subfigure}
  \caption{Representative \BOLTGAN samples after $20$ epochs on CIFAR-10, CelebA-64, LSUN Bedroom-64, and LSUN Church-64. }
  \label{fig:main_qualitative}
\end{figure}

\subsection{\texorpdfstring{Gradient-Penalty Sensitivity $\lambda_{\rm GP}$}{Gradient-Penalty Sensitivity lambda GP}}
\label{sec:G1}
We provide the full learning curves for the gradient-penalty sweep on CIFAR-10 and CelebA-64. The prior ablation is summarized in Table~\ref{tab:stability-ablation}; here the focus is only on the sensitivity to $\lambda_{\rm GP}$.

\begin{figure}[!ht]
  \centering
  \begin{subfigure}[t]{0.48\linewidth}
    \vspace{0pt}
    \centering
    \resizebox{0.98\linewidth}{!}{
\begin{tikzpicture}
\begin{axis}[
  width=3.2in, height=2.2in,
  xlabel={Epoch}, ylabel={FID $\downarrow$ (CIFAR-10)},
  xmin=0, xmax=30,              
  xtick={0,5,10,15,20,25,30},   
ymin=30, ymax=130,
ytick={30,50,70,90,110,130},
  grid=both,
  legend pos=north east,
  legend cell align=left,
]

\addplot[name path=lam1, thick, dashed, color=orange!75!black, mark=*]
table[row sep=\\]{
epoch fid_mean\\
1  111.37\\
5   60.78\\
10  49.44\\
15  45.83\\
20  44.33\\
25  40.52\\
30  40.29\\
};

\addplot[name path=lam5, thick, color=red!75!black, mark=square*]
table[row sep=\\]{
epoch fid_mean\\
1  114.28\\
5   60.67\\
10  50.37\\
15  45.95\\
20  43.61\\
25  41.03\\
30  39.53\\
};

\addplot[name path=lam10, thick, dashdotted, color=purple!75!black, mark=triangle*]
table[row sep=\\]{
epoch fid_mean\\
1  106.63\\
5   58.77\\
10  53.19\\
15  46.34\\
20  44.25\\
25  41.48\\
30  41.08\\
};

\addplot[name path=lam20, thick, color=pink!60!black, mark=diamond*]
table[row sep=\\]{
epoch fid_mean\\
1  125.73\\
5   61.67\\
10  51.43\\
15  49.46\\
20  46.75\\
25  43.71\\
30  41.35\\
};

\legend{$\lambda_{\text{GP}}=1$,$\lambda_{\text{GP}}=5$,
        $\lambda_{\text{GP}}=10$,$\lambda_{\text{GP}}=20$}
\end{axis}
\end{tikzpicture}}
    \caption{CIFAR-10}
    \label{fig:lambda_sweep_cifar}
  \end{subfigure}\hfill
  \begin{subfigure}[t]{0.48\linewidth}
    \vspace{0pt}
    \centering
    \resizebox{0.98\linewidth}{!}{
\begin{tikzpicture}
\begin{axis}[
  width=3.2in, height=2.2in,
  xlabel={Epoch}, ylabel={FID $\downarrow$ (CelebA-64)},
  xmin=0, xmax=30,              
  xtick={0,5,10,15,20,25,30},   
  ymin=5, ymax=50,
  ytick={5,10,15,20,25,30,35,40,45,50},
  grid=both,
  legend pos=north east,
  legend cell align=left,
]

\addplot[name path=lam1, thick, dashed, color=orange!75!black, mark=*]
table[row sep=\\]{
epoch fid_mean\\
1  40.99\\
5   14.04\\
10  10.69\\
15   9.95\\
20  10.53\\
25  9.19\\
30  9.10\\
};

\addplot[name path=lam5, thick, color=red!75!black, mark=square*]
table[row sep=\\]{
epoch fid_mean\\
1 42.80\\
5   14.03\\
10  11.53\\
15  10.59\\
20  9.37\\
25  9.29\\
30  9.32\\
};

\addplot[name path=lam10, thick, dashdotted, color=purple!75!black, mark=triangle*]
table[row sep=\\]{
epoch fid_mean\\
1  40.66\\
5   14.81\\
10 12.12\\
15  11.33\\
20  9.23\\
25  9.25\\
30  8.96\\
};

\addplot[name path=lam20, thick, color=pink!60!black, mark=diamond*]
table[row sep=\\]{
epoch fid_mean\\
1  45.79\\
5   15.26\\
10  12.44\\
15  10.26\\
20  9.95\\
25  9.26\\
30  9.09\\
};

\legend{$\lambda_{\text{GP}}=1$,$\lambda_{\text{GP}}=5$,
        $\lambda_{\text{GP}}=10$,$\lambda_{\text{GP}}=20$}
\end{axis}
\end{tikzpicture}}
    \caption{CelebA-64}
    \label{fig:lambda_sweep_celeba}
  \end{subfigure}
  \caption[Gradient-penalty sensitivity for BOLT-GAN]{Gradient-penalty sensitivity for \BOLTGAN. Curves show mean FID trajectories for different $\lambda_{\rm GP}$.}
  \label{fig:lambda_sweep_fid}
\end{figure}

\subsection{Extended Improved Precision and Recall}
\label{sec:G2}
The main paper reports a compact 20-epoch Precision/Recall summary for the controlled \WGANGP/\BOLTGAN comparison. Table~\ref{tab:pr} extends that comparison by including the 100-epoch results where available across the four standard datasets (best per row in bold).

\begin{table}[!ht]
\centering
\caption{Improved Precision/Recall ($\uparrow$) with $n=2000$, $k=3$.}
\label{tab:pr}
\small
\setlength{\tabcolsep}{3pt}
\renewcommand{\arraystretch}{1.15}
\begin{tabular}{l c l c c}
\toprule
Dataset & Epochs & Method & Precision & Recall \\
\midrule
CIFAR-10 & 20  & \BOLTGAN & $\mathbf{0.763}$ & $\mathbf{0.352}$ \\
         & 20  & \WGANGP  & $0.714$ & $0.183$ \\
         & 100 & \BOLTGAN & $\mathbf{0.812}$ & $\mathbf{0.469}$ \\
         & 100 & \WGANGP  & $0.764$ & $0.102$ \\
\midrule
LSUN Church-64 & 20  & \BOLTGAN & $\mathbf{0.750}$ & $\mathbf{0.411}$ \\
               & 20  & \WGANGP  & $0.681$ & $0.176$ \\
               & 100 & \BOLTGAN & $\mathbf{0.792}$ & $\mathbf{0.502}$ \\
               & 100 & \WGANGP  & $0.731$ & $0.098$ \\
\midrule
LSUN Bedroom-64 & 20 & \BOLTGAN & $\mathbf{0.598}$ & $\mathbf{0.450}$ \\
                & 20 & \WGANGP  & $0.388$ & $0.098$ \\
\midrule
CelebA-64 & 20 & \BOLTGAN & $\mathbf{0.840}$ & $\mathbf{0.548}$ \\
          & 20 & \WGANGP  & $0.672$ & $0.104$ \\
\bottomrule
\end{tabular}
\end{table}

\subsection{Synthetic 25-Gaussians: Mode Coverage and Histogram Discrepancy}
\label{sec:G-25g}
We use a $5\!\times\!5$ Gaussian-grid mixture in $\mathbb{R}^2$ with mode spacing $1$ and per-mode std $\sigma{=}0.05$ (the standard toy benchmark of \citealp{gulrajani2017improved}). Generator and discriminator are both $3$-layer MLPs of width $128$ with LeakyReLU$(0.2)$. The optimizer settings, $n_{\mathrm{critic}}=5$, and $\lambda_{\rm GP}=10$ follow the main comparison; this toy study uses batch size $256$ rather than the image default $64$. A histogram-TV diagnostic is computed on an $80\!\times\!80$ grid over $[-3,3]^2$ with $5{,}000$ samples each. Recovered modes are counted as the number of ground-truth centers with $\ge\!1$ generated sample within Euclidean radius $0.15$. Each method is run for $30$k iterations across $3$ random seeds.

\paragraph{Findings across 3 seeds.}
Under the radius-$0.15$ criterion, BOLT-GAN and WGAN-GP each cover all $25$ modes, while NS-GAN ends at $\{8,5,5\}$. The summary $6.0\pm1.4$ uses the population standard deviation across these three counts (the sample SD is $1.7$). Terminal histogram discrepancies are $0.96\pm0.01$ for BOLT-GAN/WGAN-GP and $0.91\pm0.02$ for NS-GAN. The recovered-mode count measures coverage by checking for at least one sample within $0.15$ of each center; the histogram discrepancy measures the difference between the empirical bin masses.

Figure~\ref{fig:25g} presents representative sample images that illustrate diffuse coverage for BOLT-GAN/WGAN-GP and a narrower NS-GAN support. The histogram discrepancy is evaluated on a finite grid, with its sample-size and grid-resolution effects examined below.

\begin{figure}[!htbp]
\centering
\begin{subfigure}[t]{0.48\linewidth}
\centering\resizebox{\linewidth}{!}{
\begin{tikzpicture}
\begin{axis}[
  width=\linewidth,
  height=1.75in,
  xlabel={iters},
  ylabel={recovered modes},
  xmin=0, xmax=30000,
  ymin=4, ymax=26,
  xtick={0,5000,10000,15000,20000,25000,30000},
  ytick={5,10,15,20,25},
  grid=both,
  legend pos=south east,
  legend cell align=left,
]
\addplot[thick, blue] coordinates {
  (0,8) (500,25) (3000,25) (6000,25) (9000,25) (12000,25) (30000,25)
};
\addlegendentry{BOLT-GAN}

\addplot[thick, orange] coordinates {
  (0,8) (500,25) (3000,25) (6000,25) (9000,25) (12000,25) (30000,25)
};
\addlegendentry{WGAN-GP}

\addplot[thick, green!60!black] coordinates {
  (0,8) (500,25) (12000,25) (12800,24) (13200,25) (13800,23)
  (14400,22) (15000,15) (15600,9) (16200,6) (16800,5) (30000,5)
};
\addlegendentry{NS-GAN}

\addplot[black, dashed] coordinates {(0,25) (30000,25)};
\addlegendentry{all modes}
\end{axis}
\end{tikzpicture}}\caption{Mode-coverage trajectory.}
\end{subfigure}\hfill
\begin{subfigure}[t]{0.48\linewidth}
\centering\resizebox{\linewidth}{!}{
\begin{tikzpicture}
\begin{axis}[
  width=\linewidth,
  height=1.75in,
  xlabel={iters},
  ylabel={empirical TV},
  xmin=0, xmax=30000,
  ymin=0.875, ymax=0.982,
  xtick={0,5000,10000,15000,20000,25000,30000},
  ytick={0.88,0.90,0.92,0.94,0.96,0.98},
  grid=both,
  legend pos=south west,
  legend cell align=left,
]
\addplot[thick, blue] coordinates {
  (0,0.978) (1000,0.970) (2500,0.969) (4000,0.972) (6000,0.970)
  (8000,0.969) (10000,0.965) (12000,0.963) (14000,0.964)
  (16000,0.963) (18000,0.960) (20000,0.962) (22000,0.958)
  (24000,0.957) (26000,0.955) (28000,0.951) (30000,0.947)
};
\addlegendentry{BOLT-GAN}

\addplot[thick, orange] coordinates {
  (0,0.977) (1000,0.970) (2500,0.972) (4000,0.969) (6000,0.970)
  (8000,0.970) (10000,0.971) (12000,0.970) (14000,0.968)
  (16000,0.969) (18000,0.970) (20000,0.968) (22000,0.969)
  (24000,0.967) (26000,0.969) (28000,0.964) (30000,0.966)
};
\addlegendentry{WGAN-GP}

\addplot[thick, green!60!black] coordinates {
  (0,0.977) (1000,0.969) (2500,0.972) (4000,0.970) (6000,0.971)
  (8000,0.969) (10000,0.972) (12000,0.974) (14000,0.970)
  (15000,0.950) (16000,0.935) (17000,0.936) (18000,0.924)
  (19000,0.943) (20000,0.913) (21000,0.918) (22000,0.906)
  (23000,0.918) (24000,0.894) (25000,0.910) (26000,0.892)
  (27000,0.886) (28000,0.890) (29000,0.883) (30000,0.879)
};
\addlegendentry{NS-GAN}
\end{axis}
\end{tikzpicture}}\caption{Finite-grid histogram discrepancy.}
\end{subfigure}
\par\medskip
\begin{subfigure}[t]{0.31\linewidth}
\centering\includegraphics[width=\linewidth]{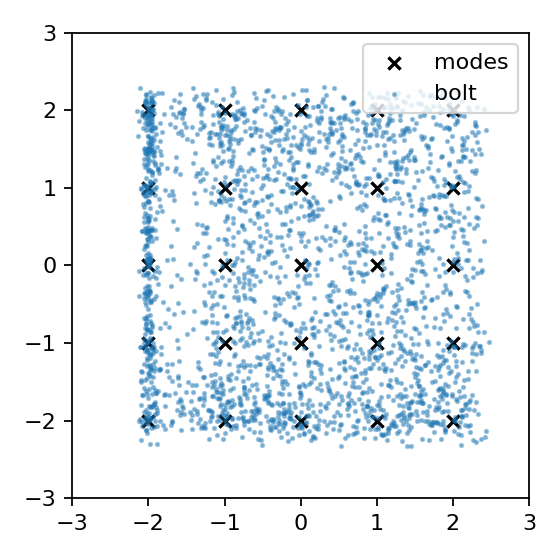}\caption{BOLT-GAN.}
\end{subfigure}\hfill
\begin{subfigure}[t]{0.31\linewidth}
\centering\includegraphics[width=\linewidth]{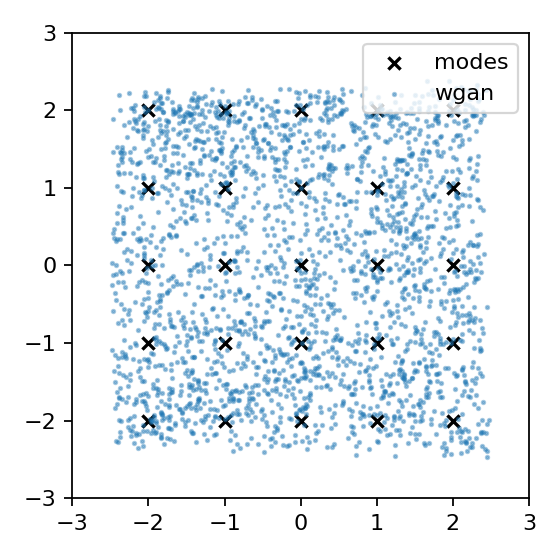}\caption{WGAN-GP.}
\end{subfigure}\hfill
\begin{subfigure}[t]{0.31\linewidth}
\centering\includegraphics[width=\linewidth]{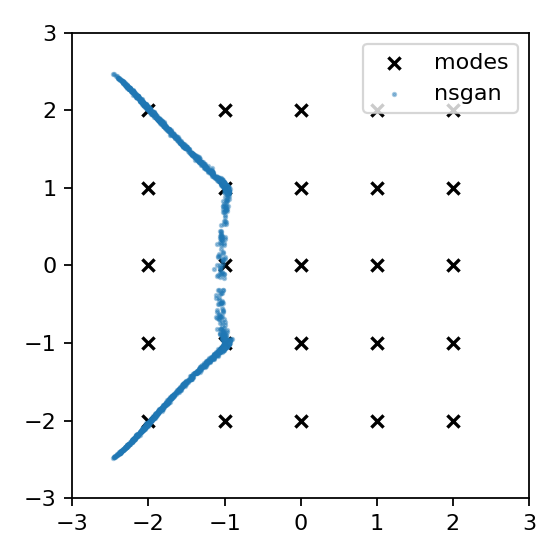}\caption{NS-GAN.}
\end{subfigure}
\caption{Synthetic 25-Gaussians: mode-coverage and histogram-discrepancy trajectories with representative samples. A mode is counted when at least one generated sample lies within radius $0.15$ of its center; the histogram discrepancy compares empirical bin masses.}
\label{fig:25g}
\end{figure}

\paragraph{Independent histogram null calibration.}
We additionally generated two independent draws of $5{,}000$ samples each from the same uniform mixture over centers $\{-2,-1,0,1,2\}^2$, with per-coordinate standard deviation $0.05$. Their population TV is exactly zero. Table~\ref{tab:histogram-null} reports $1{,}000$ independent repetitions, using equal-width grids on $[-3,3]^2$ and a separate overflow category, normalized by the total sample size. The statistic is $\frac12\sum_b|N_b/N-M_b/N|$. These target-versus-target simulations calibrate finite-sample histogram fluctuations under identical distributions.

\begin{table}[!htbp]
\centering\small
\caption{Null histogram-TV diagnostic for identical target distributions. SD is computed across $1{,}000$ repetitions; the range contains the central $95\%$ of simulated values.}
\label{tab:histogram-null}
\begin{tabular}{lrrr}
\toprule
Grid & Mean histogram TV & Sample SD & Central $95\%$ of null draws \\
\midrule
$20\times 20$ & 0.0556 & 0.0060 & $[0.0442, 0.0680]$ \\
$40\times 40$ & 0.0829 & 0.0060 & $[0.0714, 0.0948]$ \\
$80\times 80$ & 0.1368 & 0.0061 & $[0.1252, 0.1494]$ \\
$160\times 160$ & 0.2445 & 0.0061 & $[0.2330, 0.2570]$ \\
\bottomrule
\end{tabular}

\end{table}

At the original $80\times80$ resolution, the null mean is $0.1368$, with SD $0.0061$. This finite-sample effect depends on the grid resolution and the concentration of the mixture within the evaluation box. The GAN values around $0.9$ exceed this same-distribution calibration level. We interpret these values as finite-grid histogram discrepancies, alongside the mode-coverage results.

\subsection{Gaussian Evaluation of the BER Plug-in Estimator}
\label{sec:G-ber}
This numerical study validates the conditional BER estimator in Theorem~\ref{thm:ber-plug-in} using Gaussian distributions. Let $P_1=\mathcal N(\mu\mathbf1_d,I_d)$, $P_2=\mathcal N(-\mu\mathbf1_d,I_d)$, $d=16$, $\mu=0.25$, and $q_1=q_2=1/2$. Then $U(x)=\exp(2\mu\mathbf1_d^\top x)$ and $\Bayes=\Phi(-1)=0.1586553$.

\paragraph{Fixed ratio models and independent draws.}
We evaluate the oracle $U$, the deliberately biased ratio $1.2U$, and a Gaussian ratio fitted from $512$ independent training samples per class with identity covariance known. In the third case,
\begin{equation}
\widehat U(x)=\exp(w^\top x+b),\qquad
w=\widehat m_1-\widehat m_2,\quad
b=-\tfrac12(\|\widehat m_1\|_2^2-\|\widehat m_2\|_2^2).
\end{equation}
Each fitted ratio is held fixed. For $M\in\{128,512,2048,8192\}$, we independently repeat the $P_2$ evaluation $4{,}000$ times. Under $P_2$, $w^\top X+b$ is exactly Gaussian, so sampling this one-dimensional projection is equivalent to sampling all coordinates for this estimator. Fitting and evaluation use independent random streams. Code, fitted parameters, and numerical results are provided in \texttt{validation/}.

\paragraph{Exact moments and comparison.}
For $T=w^\top X+b\sim\mathcal N(m,v)$, $a=\log(q_2/q_1)$, and $k=1,2$,
\begin{align}
\E[\min(q_1e^T,q_2)^k]
&=q_1^k e^{km+k^2v/2}\Phi\!\left(\frac{a-m-kv}{\sqrt v}\right)
 +q_2^k\Phi\!\left(\frac{m-a}{\sqrt v}\right).
\end{align}
The difference between the second moment and the squared first moment, divided by $M$, is the exact conditional variance. Table~\ref{tab:ber-validation} multiplies variances by $M$ to make the predicted $1/M$ scaling visible. The empirical-to-exact variance ratios range from $0.972$ to $1.038$; the worst-case bound $M\operatorname{Var}\le1/16$ holds for all configurations. The bound is conservative in this Gaussian study.

\begin{table}[!htbp]
\centering\small
\caption{Gaussian plug-in experiment: $4{,}000$ independent evaluation repetitions per row, with each ratio held fixed. Empirical variance uses the sample-variance convention. Exact moments and the theorem bound are shown separately.}
\label{tab:ber-validation}
\begin{tabular}{lrrrrr}
\toprule
Ratio & $M$ & Empirical mean & $M\widehat{\operatorname{Var}}$ & $M\operatorname{Var}$ & $M$ bound \\
\midrule
Oracle ratio & 128 & 0.15902 & 0.03417 & 0.03292 & 0.06250 \\
Oracle ratio & 512 & 0.15872 & 0.03259 & 0.03292 & 0.06250 \\
Oracle ratio & 2048 & 0.15872 & 0.03200 & 0.03292 & 0.06250 \\
Oracle ratio & 8192 & 0.15866 & 0.03267 & 0.03292 & 0.06250 \\
$1.2\,U$ & 128 & 0.17314 & 0.03401 & 0.03495 & 0.06250 \\
$1.2\,U$ & 512 & 0.17353 & 0.03546 & 0.03495 & 0.06250 \\
$1.2\,U$ & 2048 & 0.17336 & 0.03611 & 0.03495 & 0.06250 \\
$1.2\,U$ & 8192 & 0.17344 & 0.03402 & 0.03495 & 0.06250 \\
Fitted Gaussian ratio & 128 & 0.16084 & 0.03406 & 0.03351 & 0.06250 \\
Fitted Gaussian ratio & 512 & 0.16044 & 0.03354 & 0.03351 & 0.06250 \\
Fitted Gaussian ratio & 2048 & 0.16053 & 0.03375 & 0.03351 & 0.06250 \\
Fitted Gaussian ratio & 8192 & 0.16064 & 0.03351 & 0.03351 & 0.06250 \\
\bottomrule
\end{tabular}

\end{table}

\begin{table}[!htbp]
\centering\small
\caption{Exact conditional bias and the $q_1\varepsilon_0$ bound in the same Gaussian study. Ratio errors are evaluated analytically for these exponential-affine Gaussian models.}
\label{tab:ber-bias-validation}
\begin{tabular}{lrrr}
\toprule
Ratio & Exact conditional mean & Absolute bias & $q_1\varepsilon_0$ \\
\midrule
Oracle ratio & 0.158655 & 0.000000 & 0.000000 \\
$1.2\,U$ & 0.173419 & 0.014764 & 0.100000 \\
Fitted Gaussian ratio & 0.160628 & 0.001972 & 0.100141 \\
\bottomrule
\end{tabular}

\end{table}

The fitted Gaussian ratio has absolute conditional bias $0.00197$, below its bound $0.10014$. For $1.2U$ the bias is $0.01476$, below $0.1$. For each fixed ratio, this conditional bias stays constant as $M$ increases, while the evaluation variance decreases. The oracle ratio has zero conditional bias, with finite Monte Carlo averages fluctuating around the BER. Refitting the ratio across repetitions would additionally include the ratio-fitting variability in the unconditional variance decomposition.

\subsection{Lipschitz Enforcement Ablation under a Fixed BOLT Objective}
\label{sec:G-lip}
The BOLT-GAN/WGAN-GP comparison holds GP fixed. In contrast, the HINGE-GAN comparison also changes the regularizer to SN. We isolate the regularizer by fixing the BOLT objective and sweeping the Lipschitz mechanism on CIFAR-10. All other hyperparameters match the main paper.

\begin{table}[ht]
\centering
\caption{CIFAR-10 ablation of Lipschitz enforcement with the BOLT objective fixed. FID@10k is measured after $30$ epochs and averaged over $3$ seeds; lower is better.}
\label{tab:lip-ablation}
\small
\setlength{\tabcolsep}{6pt}
\begin{tabular}{lc}
\toprule
Regularizer (BOLT objective fixed) & FID@10k $\downarrow$ \\
\midrule
None (\BOLTGANNaive)               & $\sim 380$ (diverges) \\
Weight clipping                    & $96.4$ \\
SN only                            & $51.2$ \\
GP only (paper default)            & $\mathbf{43.7}$ \\
SN + GP ($\lambda_{\rm GP}{=}2$)   & $44.0$ \\
\bottomrule
\end{tabular}
\end{table}

The three-seed mean FID values are $43.7$ for GP and $44.0$ for SN+GP. The unconstrained run again diverges, consistent with Section~\ref{sec:experiments}.

\subsection{Class-Conditional CIFAR-10 with Projection Discriminator}
\label{sec:G-cond}
We extend the controlled comparison to the class-conditional setting using a projection discriminator~\citep{miyato2018projection}. The generator is conditioned via a learnable class embedding added to the latent code; the discriminator adds an inner-product projection between embedded labels and pre-classifier features. This conditional experiment uses the same residual DCGAN family but overrides the default Adam settings: $\lambda_{\rm GP}{=}10$, $n_{\mathrm{critic}}{=}5$, Adam$(\beta_1{=}0,\beta_2{=}0.9)$, $\mathrm{lr}{=}10^{-4}$, batch size $64$, $50$ epochs.

\paragraph{Training stability.}
Under these conditional training settings, \WGANGP becomes unstable, with discriminator loss near $-10^6$ and FID above $300$. GP acts on sampled input gradients, while BOLT-GAN uses a bounded score in the adversarial objective. The sigmoid scales these gradients according to the raw score, as analyzed in Appendix~\ref{app:gradient-tradeoff}. HINGE-GAN uses a hinge objective with SN-based regularization.

\paragraph{Last versus best checkpoints.}
Table~\ref{tab:stronger-and-long} reports the last epoch-$50$ FID@10k: $44.6$ for \BOLTGAN and $63.2$ for \HingeSN. The best observed values in Figure~\ref{fig:cond-fid} are $42.6$ at epoch $45$ and $61.9$ at epoch $25$, respectively. The table therefore compares final checkpoints, while the figure also shows the best values reached during training.

\begin{figure}[!htbp]
\centering
\begin{subfigure}[t]{0.75\linewidth}
  \vspace{0pt}
  \centering
\begin{tikzpicture}
\begin{axis}[
  width=\linewidth,
  height=2.0in,
  xlabel={epoch},
  ylabel={FID@10k},
  xmin=0, xmax=52,
  ymin=40, ymax=98,
  xtick={10,20,30,40,50},
  ytick={40,50,60,70,80,90},
  grid=both,
  legend pos=north east,
  legend cell align=left,
]
\addplot[thick, blue, mark=*] coordinates {
  (5,88) (10,64) (15,61) (20,54.5) (25,50.1)
  (30,49.6) (35,46.5) (40,46.0) (45,42.6) (50,44.6)
};
\addlegendentry{BOLT-GAN}

\addplot[thick, orange, mark=*] coordinates {
  (5,96) (10,71) (15,72) (20,66) (25,61.9)
  (30,64) (35,69) (40,65.5) (45,66.5) (50,63.2)
};
\addlegendentry{HINGE-GAN}
\end{axis}
\end{tikzpicture}
  \caption{FID@10k vs. epoch.}
  \label{fig:cond-fid}
\end{subfigure}

\vspace{0.6em}
\begin{subfigure}[t]{0.48\linewidth}
  \vspace{0pt}
  \centering
  \includegraphics[width=\linewidth]{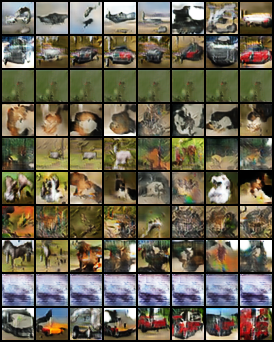}
  \caption{\HingeSN samples.}
\end{subfigure}\hfill
\begin{subfigure}[t]{0.48\linewidth}
  \vspace{0pt}
  \centering
  \includegraphics[width=\linewidth]{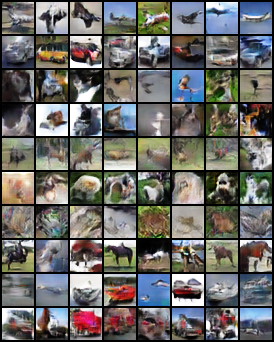}
  \caption{\BOLTGAN samples.}
\end{subfigure}
\caption{Class-conditional CIFAR-10 with a projection discriminator. Panel (a) shows that \BOLTGAN reaches a lower FID@10k than \HingeSN. Panels (b) and (c) show class-conditional samples from the same projection-discriminator architecture, with one row per class. \WGANGP is omitted because it diverges under matched hyperparameters (D loss $\to{-}10^{6}$; see text).}
\label{fig:cond-samples}
\end{figure}

\FloatBarrier
\section{Implementation Details}
\label{sec:H}
This appendix describes the experimental setup: codebase and environment, datasets and preprocessing, generator/discriminator architectures, optimization schedule, hyperparameter sweeps, evaluation metrics, and reproducibility settings. The discriminator produces a raw scalar score $f_\theta(x)\in\mathbb{R}$ and a bounded score $h_\theta(x)=\sigma(f_\theta(x))\in[0,1]$ used in the BOLT objective; all gradient penalties are applied to the raw score $f_\theta$.

\subsection{Codebase, Environment, and Hardware}
\label{subsec: env}
The neural experiments used Python~3.10+, CUDA-enabled PyTorch, automatic mixed precision (AMP), and NVIDIA A100 GPUs. The histogram null calibration in Appendix~\ref{sec:G-25g} and the Gaussian BER experiment in Appendix~\ref{sec:G-ber} use NumPy and SciPy on a CPU. Their software versions and seeds are provided in \texttt{validation/}.

\subsection{Datasets and Preprocessing}
\label{subsec: data}
We use minimal preprocessing to keep FID comparisons with prior work as consistent as possible. Unless otherwise stated, training uses unaugmented images. For CIFAR-10~\citep{krizhevsky2009learning}, we train at the native \(32{\times}32\) resolution and normalize images to \([-1,1]\); for visualization only, we upsample to \(64{\times}64\) using nearest-neighbor interpolation. For CelebA-64~\citep{liu2015face}, we center-crop faces, resize to \(64{\times}64\), and normalize to \([-1,1]\). For LSUN Bedroom/Church-64~\citep{yu2015lsun}, we apply a square center-crop, resize to \(64{\times}64\), and normalize to \([-1,1]\).

The extended experiments in Table~\ref{tab:stronger-and-long} additionally evaluate ImageNet at $32\times32$ resolution~\citep{deng2009imagenet}, Tiny ImageNet~\citep{stanford2015tinyimagenet}, and STL-10~\citep{coates2011analysis}.

\subsection{Architectures}
\label{subsec: arch}
We adopt a residual DCGAN-style backbone for the $64{\times}64$ experiments. The default critic omits BatchNorm to avoid interactions with gradient penalties. Appendix~\ref{sec:F} describes how residual and pooling operations contribute to global Lipschitz bounds.

\paragraph{Generator \(g_\varphi\) (latent \(\dim z=64\)).}
Let $\mathrm{GF}$ denote the base generator channel count. A linear layer projects \(z\sim\mathcal{N}(0,I_{64})\) to a \(4{\times}4{\times}(4\cdot\mathrm{GF})\) feature map. We then apply a stack of upsampling residual blocks whose depth depends on the target resolution: four blocks for \(64{\times}64\) (CelebA and LSUN) and three blocks for \(32{\times}32\) (CIFAR-10).

\paragraph{Discriminator (critic) \(f_\theta\).}
The critic outputs a raw scalar score $f_\theta(x)$. It receives images at dataset resolution and applies a stack of downsampling residual blocks mirroring the generator. For \(64{\times}64\), we use four downsampling residual blocks. Each block uses two \(3{\times}3\) convolutions and LeakyReLU activations with negative slope \(0.2\); the skip path uses a stride-2 \(1{\times}1\) projection. For \(32{\times}32\), we use three downsampling residual blocks of the same design. A global sum-pool feeds a linear head that outputs the raw scalar \(f_\theta(x)\). Unless otherwise noted, the critic uses no BatchNorm (i.e., we only consider instance or layer normalization in explicit ablations).

\subsection{Stronger Discriminator Architectures}
\label{subsec:stronger-discriminators}
The stronger-discriminator experiments in Table~\ref{tab:stronger-and-long} use residual CNN and ViT-style architectures adapted to output a scalar real--fake score. In each comparison, WGAN-GP and \BOLTGAN share the same generator, discriminator architecture, optimizer, batch size, number of critic updates, gradient-penalty coefficient, and evaluation protocol. The adversarial objective is the variable under comparison.

\paragraph{Residual CNN discriminator.}
The residual CNN discriminator begins with a residual downsampling block, followed by residual blocks that progressively reduce the spatial resolution. Each block uses two \(3{\times}3\) convolutional layers and a skip connection; downsampling is performed by average pooling, with a \(1{\times}1\) projection on the skip path when the number of channels changes. After the last residual block, we apply a ReLU activation, global sum pooling, and a linear layer to produce the raw scalar score \(f_\theta(x)\). 

\paragraph{ViT-style discriminator.}
The ViT-style discriminator divides the image into non-overlapping patches and maps them to token embeddings using a convolutional patch projection. A learnable class token and positional embeddings are added, and the sequence is processed by a stack of Transformer blocks with multi-head self-attention and MLP layers. The final class-token representation is passed through a linear head to produce the raw scalar score \(f_\theta(x)\).

\paragraph{Lipschitz regularization for stronger discriminators.}
For both stronger discriminator variants, the gradient penalty is applied to the raw score \(f_\theta\), and the sigmoid transformation \(h_\theta=\sigma(f_\theta)\) is used inside the BOLT objective. This keeps the raw-score penalty mechanism fixed while allowing the sigmoid to rescale the adversarial gradients and the relative strength of regularization.

\subsection{Objectives and Priors}
\label{subsec: obj}
We estimate the prior-weighted BOLT objective on mini-batches using the bounded score \(h=\sigma(f)\). The critic maximizes the BOLT gap, whereas the generator minimizes it. We set the real/fake prior to \(\pi=0.5\) by default. For \(\{x_j\}_{j=1}^B\sim P_{\mathrm{data}}\) and \(\{\hat x_j\}_{j=1}^B\sim P_{\mathrm{g}}\), let \(f_j=f_\theta(x_j)\), \(\hat f_j=f_\theta(\hat x_j)\), \(h_j=\sigma(f_j)\), and \(\hat h_j=\sigma(\hat f_j)\). The mini-batch estimator is then
\[
\widehat{\mathcal L}_{\rm MB}^{(\pi)}(g_\varphi,h_\theta)
= \pi\,\frac{1}{B}\sum_{j=1}^B h_j \;-\; (1-\pi)\,\frac{1}{B}\sum_{j=1}^B \hat h_j.
\]

\subsection{Lipschitz Enforcement}
\label{subsec: lip}
We encourage approximate \(1\)-Lipschitz behavior of the critic using the WGAN-GP gradient penalty applied to the \emph{raw} score \(f_\theta\). Concretely, we sample interpolants \(\tilde x=\alpha x+(1-\alpha)\hat x\) with \(\alpha\sim\mathrm{Unif}[0,1]\) and add
\[
\mathrm{GP} \;=\; \frac{\lambda_{\rm{GP}}}{B}\sum_{j=1}^B \Big(\,\|\nabla_{\tilde x_j}f_\theta(\tilde x_j)\|_2 - 1\Big)^2
\]
to the critic objective. We set \(\lambda_{\rm{GP}}=10\) as the common comparison setting following WGAN-GP~\citep{gulrajani2017improved}, with GP applied at every critic update. Section~\ref{sec:G1} explores sensitivity to \(\lambda_{\rm{GP}}\in\{1,5,10,20\}\).

\subsection{Optimization and Training Schedule}
\label{subsec: opt}
The default configuration uses Adam for both networks with \(\beta_1=0.5\), \(\beta_2=0.999\), constant learning rates \(\eta_{\mathrm d}=\eta_{\mathrm g}=2\times10^{-4}\), and effective batch size \(B=64\). Each iteration performs five critic updates followed by one generator update. The conditional CIFAR-10 settings in Appendix~\ref{sec:G-cond} override the Adam parameters and learning rate; the toy mixture in Appendix~\ref{sec:G-25g} uses batch size $256$.

The default configuration disables exponential moving averaging (EMA). EMA with decay $0.999$ and gradient accumulation at a fixed effective batch size are optional implementation settings. Unless otherwise specified, figures show the 20-epoch checkpoint.

\subsection{Training Algorithm}
\label{subsec: algo}
Each iteration draws a batch of real samples and a batch of latent codes, produces fake samples, and computes both raw scores \(f\) and bounded scores \(h=\sigma(f)\). We then form the mini-batch BOLT objective, evaluate the gradient penalty along line-segment interpolants, and update the critic by maximizing the BOLT gap (equivalently, minimizing $-\widehat{\mathcal L}_{\rm MB}^{(\pi)}+\mathrm{GP}$). Next, we draw a fresh latent batch and update the generator to minimize the same BOLT gap by increasing the bounded scores assigned to generated samples. Algorithm~\ref{alg:bolt-gan} uses GP at every critic step, detached fake samples for the critic update, and a fresh differentiable latent batch with fixed critic parameters for the generator update. $\operatorname{AdamStep}$ includes the separate optimizer state for each network.

\begin{algorithm}[tb]
  \caption{BOLT-GAN training with gradient penalty}
  \label{alg:bolt-gan}
  \begin{algorithmic}[1]
    \STATE {\bfseries Require:} discriminator (raw score) $f_\theta$, generator $g_{\varphi}$
    \STATE {\bfseries Require:} data distribution $P_{\mathrm{data}}$, latent distribution $P_Z$
    \STATE {\bfseries Require:} prior $\pi\in(0,1)$, batch size $B$, critic updates per generator update $n_{\mathrm d}$
    \STATE {\bfseries Require:} learning rates $\eta_{\mathrm d},\eta_{\mathrm g}$, Adam parameters $(\beta_1,\beta_2)$, iterations $T$, GP coefficient $\lambda_{\rm GP}$
    \STATE {\bfseries Ensure:} trained parameters $\theta,\varphi$
    \FOR{$t=1$ {\bfseries to} $T$}
      \FOR{$k=1$ {\bfseries to} $n_{\mathrm d}$}
        \STATE Sample $\{x_j\}_{j=1}^B \sim P_{\mathrm{data}}$ and $\{z_j\}_{j=1}^B \sim P_Z$; set $\hat x_j \gets \operatorname{stopgrad}(g_{\varphi}(z_j))$
        \STATE Raw scores: $f_j \gets f_\theta(x_j)$, $\hat f_j \gets f_\theta(\hat x_j)$
        \STATE Bounded scores: $h_j \gets \sigma(f_j)$, $\hat h_j \gets \sigma(\hat f_j)$
        \STATE {\bfseries (a) Batch estimator of $\widehat{\mathcal L}_{\rm MB}^{(\pi)}$ (fixed $\varphi$):}
        \STATE \hspace{1em}$\widehat{\mathcal L}_{\rm MB}^{(\pi)}(g_\varphi,h_\theta)\gets \pi\cdot \frac1B\sum_{j=1}^B h_j - (1-\pi)\cdot \frac1B\sum_{j=1}^B \hat h_j$
        \STATE {\bfseries (b) Gradient penalty on raw score (WGAN-GP style):}
        \STATE Sample $\{\alpha_j\}_{j=1}^B\sim \mathrm{Uniform}(0,1)$; set $\tilde x_j \gets \alpha_j x_j + (1-\alpha_j)\hat x_j$
        \STATE \hspace{1em}${\rm GP}\gets \frac{\lambda_{\rm GP}}{B}\sum_{j=1}^B \big(\|\nabla_{\tilde x_j}f_\theta(\tilde x_j)\|_2-1\big)^2$
        \STATE {\bfseries (c) Critic update (maximize $\widehat{\mathcal L}_{\rm MB}^{(\pi)}$):}
        \STATE \hspace{1em}$\mathcal{L}_\theta \gets -\,\widehat{\mathcal L}_{\rm MB}^{(\pi)}(g_\varphi,h_\theta) + {\rm GP}$
        \STATE \hspace{1em}$\theta\leftarrow\operatorname{AdamStep}(\theta,\nabla_\theta\mathcal L_\theta;\eta_{\mathrm d},\beta_1,\beta_2)$
      \ENDFOR
      \STATE {\bfseries Generator update (minimize $\widehat{\mathcal L}_{\rm MB}^{(\pi)}$):}
      \STATE Sample $\{z'_j\}_{j=1}^B \sim P_Z$; set $\hat x'_j \gets g_{\varphi}(z'_j)$; $\hat h'_j \gets \sigma\!\big(f_\theta(\hat x'_j)\big)$
      \STATE \hspace{1em}$\mathcal{L}_\varphi \gets -\, (1-\pi)\cdot \frac1B\sum_{j=1}^B \hat h'_j$
      \STATE \hspace{1em}$\varphi\leftarrow\operatorname{AdamStep}(\varphi,\nabla_\varphi\mathcal L_\varphi;\eta_{\mathrm g},\beta_1,\beta_2)$
    \ENDFOR
  \end{algorithmic}
\end{algorithm}
\subsection{Hyperparameter Sweeps and Selection}
\label{subsec: sweeps}
Table~\ref{tab:fid20} uses the common default $\lambda_{\rm GP}=10$ for both methods. The sweep-best results are reported in Table~\ref{tab:stability-ablation}. The separate CIFAR-10 sweep in Table~\ref{tab:stability-ablation} evaluates $\lambda_{\rm GP}\in\{1,5,10,20\}$ at a fixed budget. Its best values occur at $5$: FID $43.6$ for BOLT-GAN and $59.6$ for WGAN-GP.  Conditional last/best checkpoints are specified in Appendix~\ref{sec:G-cond}.

\subsection{Evaluation Metrics}
\label{subsec:eval-metrics}
\paragraph{Fréchet Inception Distance (FID).}
We use FID as the primary quantitative measure of image-generation quality. FID compares Gaussian distributions fitted to the Inception-V3 features of real and generated images~\citep{heusel2017gans}. Lower values indicate a smaller discrepancy between these Gaussian feature summaries. Unless noted otherwise, all FID scores are computed on 10k generated samples using the precomputed real-data Inception statistics for each dataset.

\paragraph{Kernel Inception Distance (KID) and Inception Score (IS).}
KID estimates squared maximum mean discrepancy between real and generated Inception features using a polynomial kernel~\citep{binkowski2018demystifying}. Lower values indicate a smaller feature discrepancy. IS summarizes the divergence between the Inception class predictions for individual generated images and their marginal distribution~\citep{salimans2016improved}; higher values are preferred. Table~\ref{tab:stronger-and-long} reports KID multiplied by $10^3$.

\paragraph{Evaluation Protocol.}
Unless stated otherwise, we evaluate every checkpoint with 10k samples and the precomputed real-data Inception statistics. Table~\ref{tab:fid20} uses the common 20-epoch checkpoint and $\lambda_{\rm GP}=10$. Appendix~\ref{subsec: sweeps} gives the GP-coefficient sweep; Appendix~\ref{sec:G-cond} distinguishes final and best-observed conditional checkpoints.

\paragraph{Metric reporting.}
The prior sweep in Table~\ref{tab:stability-ablation} reports three-seed mean$\pm$SD, while the main comparisons in Tables~\ref{tab:fid20} and~\ref{tab:stronger-and-long} use point estimates. Figure~\ref{fig:fid_curve} shows mean FID trajectories over three seeds; Tables~\ref{tab:fid20} and~\ref{tab:stronger-and-long} summarize their stated training budgets. At $100$ epochs on CIFAR-10, $34.9$ is the FID@50k summary and $36.6\pm0.8$ is the three-seed FID@10k summary. 

\subsection{Logging and Diagnostics}
\label{subsec: logging}
Local training diagnostics include sampled $\|\nabla_{\tilde x}f_\theta(\tilde x)\|_2$, the raw adversarial term, weighted GP, and the stabilized ratio $\log(1+r_{\rm GP})$ in \eqref{eq:stable-gp-diagnostic}. The chosen $\epsilon$ specifies the ratio's scale, and checkpoint/configuration identifiers connect diagnostic trajectories to reported metrics.

\subsection{Determinism and Seeds}
\label{subsec: determinism}
For deterministic runs, we seed Python, NumPy, PyTorch, and dataloader workers, with \texttt{torch.backends.cudnn.deterministic=True} and \texttt{torch.backends.cudnn.benchmark=False}. Exact replication additionally depends on the hardware and software environment, precision mode, random seeds, configuration, checkpoint-selection rule, and code revision.

\subsection{Failure Modes and Stability Notes}
\label{subsec: fail}
The non-Lipschitz BOLT variant (\(\lambda_{\rm{GP}}=0\)) frequently diverges or produces very high FID, which is consistent with optimizing a TV-strength objective. Combining strong spectral normalization with a large gradient-penalty coefficient can over-regularize the critic; in such cases we prefer to rely primarily on a single Lipschitz control. We also avoid BatchNorm in the critic. Critic saturation motivates tuning the discriminator learning rate, update ratio, and regularization scale on a prespecified validation budget.

\paragraph{Code availability.}
Code is available at our \href{https://anonymous.4open.science/r/bolt-gan-7790}{anonymous repository}.

\end{document}